%% file: main.tex
\documentclass[12pt]{iopart}
\usepackage{iopams}  

\RequirePackage{amsthm}
\RequirePackage[authoryear]{natbib}
\RequirePackage[colorlinks,citecolor=blue,urlcolor=blue]{hyperref}%% uncomment this for coloring bibliography citations and linked URLs
\RequirePackage{graphicx}%% uncomment this for including figures
\usepackage{aas_macros}
\usepackage[utf8]{inputenc} % allow utf-8 input
\usepackage[T1]{fontenc}    % use 8-bit T1 fonts
\usepackage{hyperref}       % hyperlinks
\usepackage{url}            % simple URL typesetting
\usepackage{booktabs}       % professional-quality tables
\usepackage{nicefrac}       % compact symbols for 1/2, etc.
\usepackage{microtype}      % microtypography
\usepackage{xcolor}         % colors
\usepackage{subcaption}
\usepackage{stackrel}

\usepackage{booktabs} % commands to create good-looking tables
\usepackage{tikz} % nice language for creating drawings and diagrams
\usepackage{algorithm}
\usepackage{algpseudocode}
\usepackage[export]{adjustbox}
\usepackage{iopams}
\usepackage{xspace}
\usepackage{tikz-cd}
\usepackage{float}
\usepackage{xcolor}
\usepackage[cases]{cases}
\usepackage{siunitx}

\newrobustcmd\B{\DeclareFontSeriesDefault[rm]{bf}{b}\bfseries}
\newcommand\mcc[1]{\multicolumn{1}{c}{#1}}

\definecolor{red}{rgb}{0.64705882, 0., 0.14901961}
\definecolor{blue}{rgb}{0.19215686, 0.21176471, 0.58431373}

\newtheorem{thm}{Theorem}
\newtheorem{Assumption}{Assumption}
\newtheorem{Lemma}{Lemma}
\newtheorem{Corollary}{Corollary}
\newtheorem{Definition}{Definition}
\newtheorem{Remark}{Remark}

\def\I{{\mathbb I}}
\def\P{{\mathbb P}}
\def\pr{\mathbb{P}}

\def\X{{\mathbf{X}}}
\def\x{{\mathbf{x}}}
\def\s{{\mathbf{s}}}

\def\y{{\mathbf{y}}}
\def\Y{{\mathbf{Y}}}

\def\E{{\mathbb E}}

\def\PIT{{\rm PIT}}

\let\hat\widehat

\renewcommand{\S}{\mathbf{S}}

\newcommand{\pz}{photo-$z$\xspace}
\newcommand{\pzs}{photo-$z$'s\xspace}
\newcommand{\calpit}{\texttt{Cal-PIT}\xspace}

\newcommand{\codecomment}[1]{\textbf{\color{black}// #1}}

\definecolor{awesome}{rgb}{1.0, 0.13, 0.32}
\definecolor{safetyorange}{rgb}{1.0, 0.4, 0.0}
\definecolor{vermilion}{rgb}{0.89, 0.26, 0.2}
\definecolor{aqua}{rgb}{0.0, 0.9, 0.9}
\definecolor{darkorchid}{rgb}{0.6, 0.2, 0.8}
\definecolor{darkblue}{rgb}{0.0, 0.0, 0.55}

\DeclareUnicodeCharacter{2212}{-}
\newcommand{\xrightarrow}[2][]{%
  \begin{tikzcd}[cramped, ampersand replacement=\&, outer sep=0pt]{}\arrow[r, "#2","#1"']\&{}\end{tikzcd}%
}

\begin{document}
\title[Local Amortized Diagnostics and Reshaping of Conditional Densities]{Towards Instance-Wise Calibration: Local Amortized  Diagnostics and Reshaping of Conditional Densities (LADaR)}

\author{Biprateep Dey$^{1,2,3,4,5,6}$, David Zhao$^7$, Brett H. Andrews$^{1,2}$, Jeffrey A. Newman$^{1,2}$, Rafael Izbicki$^8$ and Ann B. Lee$^7$\footnote[1]{{\em Corresponding author:} Ann B. Lee, annlee@andrew.cmu.edu}}
\address{$^1$Department of Physics and Astronomy, University of Pittsburgh}
\address{$^2$Pittsburgh Particle Physics, Astrophysics, and Cosmology Center (PITT PACC), University of Pittsburgh}
\address{$^3$Department of Statistical Sciences, University of Toronto}
\address{$^4$Canadian Institute for Theoretical Astrophysics (CITA), University of Toronto}
\address{$^5$Dunlap Institute for Astronomy \& Astrophysics, University of Toronto}
\address{$^6$Vector Institute}
\address{$^7$Department of Statistics and Data Science, Carnegie Mellon University}
\address{$^8$Department of Statistics, Federal University of S\~{a}o Carlos (UFSCar)}

\begin{abstract} 
Key science questions, such as galaxy distance estimation and weather forecasting, often require knowing the full predictive distribution of a target variable $y$ given complex inputs $\x$. Despite recent advances in machine learning and physics-based models, it remains challenging to assess whether an initial model is calibrated for all $\x$, and when needed, to reshape the densities of $y$ toward ``instance-wise'' calibration. This paper introduces the LADaR (Local Amortized Diagnostics and Reshaping of Conditional Densities) framework and proposes a new computationally efficient algorithm (\calpit) that produces interpretable local diagnostics and provides a mechanism for adjusting conditional density estimates (CDEs). \calpit learns a single interpretable local probability--probability map from calibration data that identifies where and how the initial model is miscalibrated across feature space, which can be used to morph CDEs such that they are well-calibrated. We illustrate the LADaR framework on synthetic examples, including probabilistic forecasting from image sequences, akin to predicting storm wind speed from satellite imagery. Our main science application involves estimating the probability density functions of galaxy distances given photometric data, where \calpit achieves better instance-wise calibration than all 11 other literature methods in a benchmark data challenge, demonstrating its utility for next-generation cosmological analyses.\footnote[2]{Code available as a Python package here: \url{https://github.com/lee-group-cmu/Cal-PIT}}

\end{abstract}


\maketitle

\section{Introduction}\label{sec:intro}
In recent decades, many scientific fields have progressed from computing point predictions (or a single best guess of a quantity of interest) to developing full predictive distributions, or more specifically,  {\em conditional density estimates (CDEs) and generative models} of a response/``target'' variable $Y \in \mathbb{R}$ given covariates/features $\X \in \mathbb{R}^d$. This paradigm shift is evident in various disciplines, such as in astrophysics (e.g., \citealt{Mandelbaum2008PhotozPDF, Malz2022Photoz}), in weather forecasting (e.g., \citealt{gneiting2008probabilistic, Ravuri2021GenerativeWeather,Li2024GenWeather}), in financial risk management (e.g., \citealt{timmermann2000PDFinance}), and in epidemiological projections (e.g., \citealt{alkema2007HIVPD}).

The paradigm shift has been driven by two main factors.  First, advances in measurement technology across engineering, physical and biological sciences are producing data with unprecedented depth, richness, and scope. To fully exploit these data in subsequent analyzes, we need precise estimates of the uncertainty in a response variable $Y$ given observable data $\X$ (see Section~\ref{sec:motivation} for two applications from the physical sciences that motivated this work). Second, we are experiencing a rapid growth of high-capacity machine learning algorithms that allow the quantification of uncertainty for complex, high-dimensional inputs of different modalities.  Two examples of such datasets come from (1) large astronomical surveys that collect images and spectroscopic data for tens of millions of stars, galaxies and other astrophysical objects \citep{York2000SDSSOverview,GAIA2016Overview,Dey2019legacySurvey,DesiColl2022DESIOverview} and (2) earth-observing satellites for environmental and climate science (see, e.g., NASA's Earth Observing System\footnote[1]{\url{https://eospso.nasa.gov/}} and next-generation Earth System Observatory,\footnote[2]{\url{https://science.nasa.gov/earth-science/missions/earth-system-observatory/}}).  For the latter, the dimension $d$ of the input space (representing, e.g., the number of image pixels or different spatial locations) is usually several orders of magnitude larger than $10^6$. In addition to enabling uncertainty quantification for complex data, modern machine learning methods allow us to ``amortize'' the computation; that is, to perform the compute-intensive training process only once, which allows for very fast inference and scaling to massive data sets.

Machine learning methods for uncertainty quantification (UQ) include a growing range of approaches. Explicit conditional density estimation (CDE) methods directly model $f(y|\x)$, using tools like mixture density networks \citep{Bishop1994MDN}, kernel mixture networks \citep{Ambrogioni2017KMN}, normalizing flows \citep{papamakarios2019normalizing,Kobyzev2021NormalizingFlow}, and other nonparametric estimators \citep{izbicki2016nonparametric,izbicki2017converting,Dalmasso2020FlexcodePhotoz}. Implicit CDEs and generative models—such as VAEs \citep{Kingma2013Vae}, conditional GANs \citep{Mirza2014cGAN}, diffusion models \citep{Sohl-Dickstein2015DiffusionModel1,Ho2020DiffusionModels2,diffusion3,diffusion4,diffusion5}, and transformer-based generators \citep{vaswani2017Transformers,radford2019language}—represent uncertainty through learned stochastic mappings. Other strategies include quantile regression \citep{Chung2021Quantile,Fasiolo2021Quantile,Lucrezia2018Quantile,feldman2021orthogQR,lim2021temporalFusionTransformer} and ensemble-based methods, such as dropout and deep ensembles \citep{gal2016dropout,lakshminarayanan2017simple,rahaman2021uncertainty}.

The goal of this paper is not to add to this list, but rather to provide the scientist with a unified interpretable framework for deciding whether an initial model of the predictive distribution is accurate with respect to (conditional on) relevant features, and if not, suggest a mechanism for reshaping CDEs. Figure~\ref{fig:schematic} shows a schematic diagram of our LADaR approach. The starting point is an initial CDE --- which could, e.g., be derived from pre-trained neural networks on massive generic data (so-called foundation models) or physics-based models such as numerical weather prediction (NWP) models. LADaR addresses three key questions: (1) Does the initial model need to be improved with respect to relevant features? (2) Where in the feature space might it need to be improved? (3) How can it be improved? For the third question, we propose a reshaping step that adjusts the initial CDEs while leveraging its existing strengths. LADaR is particularly relevant when there are insufficient observational data to independently fit a purely ML-based CDE model, or when the scientist needs to tie results to physical processes in the native feature space (defined by, e.g., individual spectra or specific sequences of satellite imagery) to trust predictions and stated uncertainties.

\begin{figure}[htb]
	\centering
	\includegraphics[trim={0 4cm 0 0},clip=true,width=0.83\textwidth ]{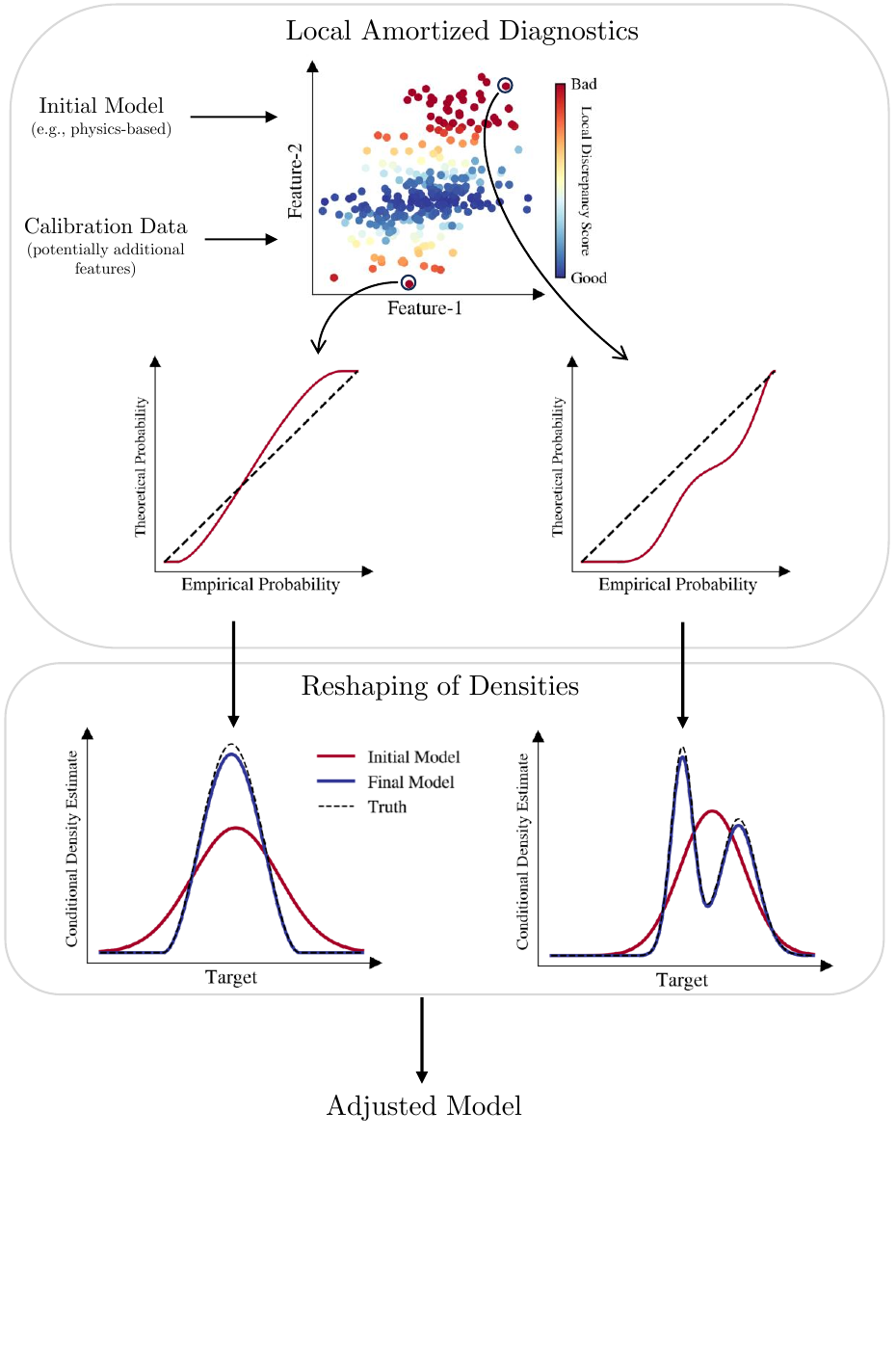}
	\caption{\footnotesize {\bf Schematic representation of the LADaR approach.} Our approach starts with an initial (e.g., physics-based or large pre-trained) model of the predictive distribution of a target quantity. We then assess the quality of the initial conditional density estimates (CDEs) on an individual basis across the feature space using calibration data, and reshape the densities if deemed necessary.  The goal is not to replace the initial model with a different end-to-end density estimator, but rather to adjust it, ensuring both calibration and insight into its potential failure modes (see Figure~\ref{fig:PPplot_interpret} for how to interpret P-P plots). The LADaR approach is particularly relevant when there are insufficient observational data to independently fit a purely machine-learning-based CDE, or when it is important to tie predictions to the underlying physical processes (encoded by the chosen feature space) to establish trust in machine-learning methods. Our framework is fully ``amortized'' over both features $\x$ and response variable $y$, which means that once we have trained LADaR to learn the map between the initial CDE model and the CDE of the calibration data, no additional training is required for new data.}
	\label{fig:schematic}
\end{figure}

\subsection{Trustworthy Uncertainty Quantification}
For a conditional density estimator  to be useful, its predicted distribution $\widehat F(y|\x)$ (with density function $\widehat{f}(y|\x)$)  must closely match the true $F(y|\x)$ for each value of the input $\x$. This property, known as {\em local or instance-wise calibration}, ensures that predicted probabilities reflect true frequencies for individual cases, and not just on average.
 
 Instance-wise UQ is essential in many practical applications; e.g., in astrophysical studies, for predicting the physical properties of individual galaxies from measured fluxes; in weather forecasts, for predicting the probability of rainfall based on current environmental conditions; and in medical research, for estimating a drug's efficacy for individuals of specific demographics. Instance-wise calibration also promotes algorithmic fairness by avoiding systematic over- or under-prediction of risks for certain groups  \citep{kleinberg2016inherent,zhao2020individual}, and enables well-calibrated prediction sets (Remark~\ref{remark:prediction_sets}).

 {\em Unfortunately, off-the-shelf CDE methods can be very far from calibrated.} 
 This is because they minimize losses that don't target calibration directly—such as KL divergence \citep{KLDivergence1951}, integral probability metrics \citep{papamakarios2019normalizing,dalmasso2020conditional}, or the pinball loss \citep{koenker2001QR}. As shown by \citet{Guo2017Calibration} and \citet{chung2021beyondpinball}, many ML methods prioritize accuracy and sharpness over calibration. To address this, new loss functions have been proposed to balance calibration and sharpness \citep{chung2021beyondpinball} or decouple coverage from sharpness \citep{feldman2021orthogQR}.

Finally, in terms of {\em diagnostics}, many common metrics for assessing calibration, like the probability integral transform (PIT; \citealt{gan1990pit}) and simulator-based calibration (SBC; \citealt{talts2018validating}), only evaluate {\em marginal} calibration—that is, average coverage over all $\x$'s (Equation~\ref{eq:marginal_calibration}). This weaker notion is often referred to simply as ``calibration''  \citep{Gneiting2014Review, kuleshov2018accurate}. However, as pointed out by \citet{Schmidt2020Photo-z}, PIT can be optimal even when the model ignores $\x$ entirely. More generally, errors across the feature space can cancel out, leading to deceptively good marginal results \citep{zhao2021diagnostics, Jitkrittum2020LocalCalibration, Luo2021LocalCallibration}. For instance, \citet[Theorem 1]{zhao2021diagnostics} showed that even models based on $F(y|g(\x))$—for any function $g$—can pass marginal tests, despite discarding relevant features.

\subsection{Well-Calibrated CDEs are Essential for the Physical Sciences}~\label{sec:motivation}  Our trustworthy CDE work is motivated by two main applications in astronomy and weather forecasting:\\

\noindent  {\em (i) Photometric redshift estimation of galaxies.}
Estimating galaxy distances, via a measurable proxy called redshift, is a fundamental task for studies of astrophysics and cosmology. While spectroscopy can precisely measure redshift, this method is too resource-intensive for the billions of galaxies detected by modern astrophysical imaging surveys, so galaxy redshifts must be predicted from imaging data alone.   In this context, the response variable $y$ is the galaxy's redshift (by convention denoted by $z$), and the predictors are photometric/imaging data $\x$.  The predictions, called photometric redshifts (\pzs), are inherently probabilistic.  Downstream science applications rely on an accurate estimate of the conditional density for each galaxy's redshift. The scientific requirements are extremely strict: to avoid biasing cosmological results, the errors in the moments of the redshift distributions for large ensembles of galaxies must be within 0.1--0.3\% of the truth \citep{DESCSRD}.

Our proposed \pz\ use case is to adopt a physics-based \pz\ model to produce initial estimates of PDFs, and then use the LADaR framework to assess the initial CDEs and reshape them if necessary.  Furthermore, the interpretability of the LADaR diagnostics will be valuable for helping astrophysicists improve both physics and machine learning-based \pz models.\\

\noindent {\em (ii) Probabilistic forecasting of the intensity of tropical cyclones (TC) from satellite imagery.} Tropical cyclones are highly organized rotating storms that are among the most costly natural disasters in the United States. TC intensity forecasts have improved in recent years, but these improvements have been relatively slow during the last decade compared to improvements in track forecasts, particularly at 24-hour lead times~\citep{IMPROVE}. The latest generation of geostationary satellites (GOES), such as GOES-16, now provides unprecedented spatio-temporal resolution of TC structure and evolution~\citep{schmit2017closer}. 
A broad range of recent work involving neural networks has explored the wealth of information from GOES imagery for TC short-term intensity prediction (e.g., \citealt{olander2021investigation, griffin2022predicting}). In this context, the response variable $Y$ is the TC's intensity (wind speed) at time $t+\tau$ for a lead time $\tau$ of up to 24 hours, and $\x$ could represent environmental predictors and a sequence of images at the current time $t$ and preceding time points.  In Section~\ref{sec:example_TCs}, we present a TC-inspired synthetic example that highlights the efficacy of our method in diagnosing and recalibrating intensity forecasts with high-dimensional sequence data as inputs.

\subsection{Our Contribution}  

To ensure reliable uncertainty quantification with CDEs, it is essential to have (i) interpretable diagnostics that can assess instance-wise calibration and failure modes of an initial model across the entire feature space of reference data, and (ii) computationally efficient methods that can reshape CDEs so that they are approximately calibrated for every $\x$.  The initial model can, for example, represent the best approximation to the true conditional density according to a physics-motivated or a mathematical model, or from a data-driven model limited to a set of easily accessible input features or data sources.

The goal is to morph the initial model towards the true distribution of the quantity of interest by leveraging calibration data and machine-learning techniques, when such an adjustment is deemed to be necessary by the scientist. This work offers two primary contributions:

\begin{itemize}
    \item  From a methodological perspective, we present a unified framework for interpretable diagnostics and reshaping of entire CDEs through a single Probability-Probability (P-P) map learned from calibration data $\mathcal{D} = \{(\X_1, Y_1), \ldots, (\X_n, Y_n)\}$, which implicitly encodes the true distribution $F(y|\x)$. Our approach is fully ``amortized,'', which means that once a regression model is trained to learn the mapping, no additional training is required for new data. We refer to the general framework of Local Amortized Diagnostics and Reshaping of CDEs as LADaR, and call our proposed algorithm \calpit. The first prototype code  of \calpit occurred in \citep{Dey2021calpitProto}; the full ready-to-use and modifiable implementation is now available as a Python package at \href{https://github.com/lee-group-cmu/Cal-PIT}{\url{https://github.com/lee-group-cmu/Cal-PIT}}.
    
    \item From an application perspective, \calpit is uniquely positioned to provide diagnostics and ensure that \pz CDEs are \textit{locally} calibrated (i.e., not only as a full ensemble), which will be necessary for the astrophysics community to meet the stringent \pz requirements for next generation-astronomical surveys. Figure \ref{fig:photo-z-local} and Table \ref{tab:photo_z_comparison} demonstrate the full potential of \calpit applied to a key benchmark \pz data set, where \calpit outperforms the current state-of-the-art for diagnostics and estimation of \pz CDEs. Crucially, \calpit can (i) accurately reshape biased probability distributions and (ii) reshape unimodal distributions into multimodal distributions---both common failure modes for common \pz estimation methods. Furthermore, \calpit has the flexibility to be used with \textit{high-dimensional} and \textit{dependent sequence} data.  Section~\ref{sec:example_TCs} shows \calpit applied to probabilistic forecasting with sequences of images as inputs, akin to predicting the wind speed of tropical cyclones (TCs) from satellite imagery. 
   
\end{itemize}

\section{Related Work}

\subsubsection*{Goodness-of-Fit Tests and Calibration.}  Goodness-of-fit of conditional density models to observed data can be assessed by two-sample tests \citep[e.g., ][]{andrews1997condtest, stute2002condtest, moreira2003condtest,jitkrittum2020cde}. 
 Such tests are useful for deciding whether a conditional density model needs to be improved, but do not provide any means to correct discrepancies. One way to recalibrate CDEs (proposed, e.g., by \citealt{bordoloi2010photoz}) is to first assess how the marginal distribution of PIT values differs from a uniform distribution by diagnostic tools \citep{cook2006validating,freeman2017photoz,talts2018validating,disanto2018cmdn}, and then apply corrections to bring them into agreement. However, by construction, such recalibration schemes only improve marginal calibration. In this work, we instead build on \citet{zhao2021diagnostics}, which proposes a version of PIT that is estimated throughout the {\em entire} input feature space, allowing us to directly assess and target conditional coverage.

\subsubsection*{Quantile Regression.} 
Quantile regression intervals converge to the oracle $C^{*}_{\alpha}(\X)=\left[F^{-1}(0.5\alpha|\X), F^{-1}(1-0.5\alpha|\X)\right]$ \citep{koenker1978QR,taylor1999quantile}. Even though the prediction interval $C^{*}_{\alpha}(\X)$ satisfies conditional validity, 
 \begin{equation*} 
  \mathbb{P}(Y\in C_\alpha(\X) |\X = \x) = 1 -\alpha, \ \ \forall \x \in \mathcal{X}, 
\end{equation*}
 the standard pinball loss can yield highly miscalibrated UQ models for finite data sets \citep{chung2021beyondpinball,feldman2021orthogQR}. New loss functions have been proposed to address this problem \citep{chung2021beyondpinball,feldman2021orthogQR}. Our approach also provides calibrated prediction regions, but is more general, yielding full CDEs and not only prediction intervals.

 \subsubsection*{Conformal Inference.} 
 Conformal prediction methods have the appealing property of producing prediction sets with finite-sample marginal validity, $\mathbb{P}(Y\in C(\X)) \geq 1 -\alpha$, as long as the data are exchangeable \citep{Vovk2005, lei2018distribution}. However, there is no guarantee that conditional validity is satisfied, even approximately. 
 More recent efforts have addressed approximate conditional validity \citep{romano2019conformalizedQR,izbicki2020Dist-Split,chernozhukov2021distributional,izbicki2022,cabezas2025regression} 
 by designing conformal scores with an approximately homogeneous distribution throughout $\mathcal{X}$.
 Unfortunately, it is difficult to check whether these methods provide good conditional coverage in practice.  Moreover, our method provides estimates of the full CDE, and not only prediction bands. Calibrated CDEs imply calibrated prediction bands, but not vice versa.

\section{Methodology}\label{sec:methodology}

{\em Objective and notation.} Our LADaR goal is to reshape an initial (often simple) cumulative distribution $\hat{F}(y|\x)$, or equivalently, its conditional density $\hat{f}(y|\x)$, to achieve approximate instance-wise calibration with respect to some implicit (often more complex but not explicitly known) target distribution $F(y|\x)$. 
Instance-wise calibration is defined as
\begin{equation} \label{eq:conditional_calibration}
\centering
  \hat{F}(y|\x) = F(y|\x), \ \ \textrm{for all } y, \textrm{at every } \x,
\end{equation}
  and is sometimes also referred to as conditional or local calibration. Instance-wise or conditional calibration implies marginal calibration
 \begin{equation} \label{eq:marginal_calibration}
  \hat{F}(y) = F(y), \ \ \textrm{for all } y,
\end{equation}
 whereas the reverse implication is not true.

To achieve instance-wise calibration, we assume the availability of i.i.d.\ calibration data  $\mathcal{D} = \{(\X_1, Y_1), \ldots, (\X_n, Y_n)\}$ from $F_{\X,Y}(\x,y)$, the joint distribution of $(\X, Y)$, where $\x \in \mathcal{X} \subseteq \mathbb{R}^d$ and $\y \in \mathcal{Y} \subseteq \mathbb{R}$. We assume that the joint distribution  is a product $F_{\X,Y}(\x, y)=F(y|\x)F(\x)$ of the target distribution $F(y|\x)$ and some sampling distribution $F(\x)$ with support over the entire feature space $\mathcal{X}$.  

 In this paper, we propose a solution to the problem of diagnosing and ensuring local calibration of conditional densities based on probability integral transforms. We refer to the algorithm and the code as \calpit. The details are as follows.

\subsection{Overview of the \calpit \ Algorithm}

The \calpit algorithm first computes interpretable diagnostics using regression that identifies the failure modes of the initial conditional density model and pinpoints the location of poorly calibrated instances in a potentially high-dimensional feature space. The same regression function used for diagnostics is then used to continuously transform the potentially misspecified CDE into a new CDE that is approximately calibrated for all $\x$.

\calpit builds on the observation that an estimate of a cumulative distribution function (CDF), $\widehat{F}$, is calibrated for every instance $\x$, if and only if its probability integral transform (PIT) value conditional on $\x$, defined by $\PIT(Y; \X) := \hat{F}(Y|\X)$, where $(\X,Y)$ is drawn from $F_{\X,Y}$, is uniformly distributed \citep[Corollary 1]{zhao2021diagnostics}. As a result, if a CDE is well-calibrated, the cumulative distribution function of the PIT (hereafter PIT-CDF), defined as the cumulative distribution of the PIT random variable  evaluated  at $\gamma \in (0,1)$,
\begin{equation}\label{eq:r_alpha}
    r^{\hat{f}}(\gamma; \x) := \pr \left( \PIT(Y; \X) \leq \gamma \ \middle| \ \x \right),
\end{equation}
will be approximately $\gamma$ for all $\x \in \mathcal{X}$ and $\gamma \in (0,1)$. In other words, the PIT-CDF will then correspond to the CDF of a uniform random variable for all $\x$. The PIT-CDF provides valuable information as to whether $\hat F$ is miscalibrated, and if so, for what instances $\x$, for what types of deviations and to what extent. Specifically,  local probability-probability (P-P) plots --- which graph the PIT-CDF value $ r^{\hat f}(\gamma;\x)$, the empirical probability, against $\gamma$, the theoretically expected probability, for fixed $\x$ ---  offer valuable information on how close the probability distribution $\hat{F}(Y|\X)$ is to $F(Y|\X)$ at different locations $X=\x$ in the feature space. Figure~\ref{fig:PPplot_interpret} shows a schematic diagram of some P-P plots and how to interpret them. 

\begin{figure}[htb]
	\centering
	\includegraphics[trim={0 15cm 0 0},clip=true,width=\textwidth ]{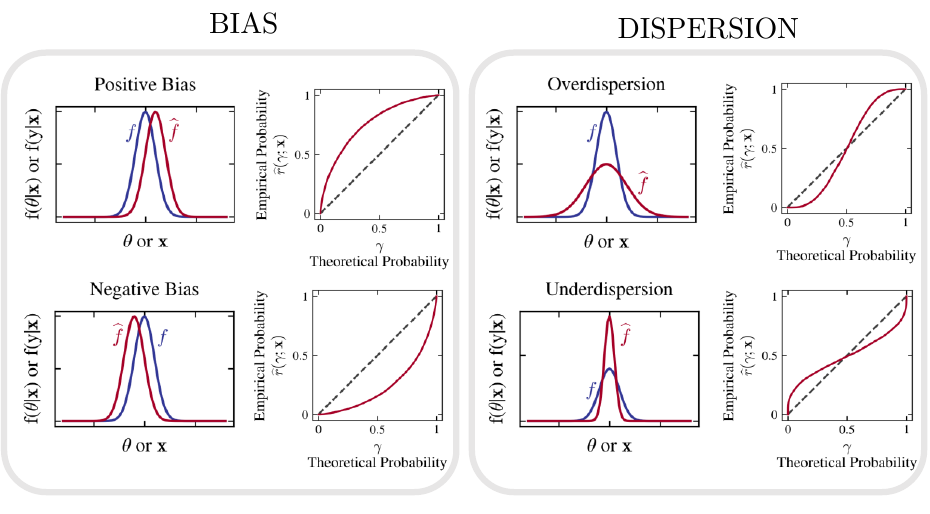}
	\caption{\footnotesize {\bf Interpretable diagnostics.} P-P plots are commonly used to assess how well
			a probability density model fits actual data. Such plots display, in a clear and interpretable way, effects like {\em bias} ({\em left panel}) and {\em dispersion} ({\em right panel}) in an estimated distribution $\widehat f$ vis-a-vis 
			the true data-generating distribution $f$. Our framework yields  an amortized approach to constructing  local P-P plots for comparing  
   Bayesian posteriors $\widehat f(\theta|\x)$ or predictive densities $\widehat f(y|\x)$ at {\em any} location $\x$ of the feature space $\mathcal{X}$. Figure adapted from \citealt{zhao2021diagnostics}. An interactive version of this figure can be found at: \url{https://lee-group-cmu.github.io/cal-pit-paper/fig_1_interactive/}.
	}
	\label{fig:PPplot_interpret}
\end{figure}

However, in practice, because the distribution of the PIT statistic depends on the true conditional distribution of $Y|\x$, the PIT-CDF is unknown. Section~\ref{sec:diagnostics} describes how one can estimate the PIT-CDF across the feature space from calibration data using a regression method suitable for the problem at hand. Our proposed approach is {\em amortized}, in the sense that one can train on $\x$ and $\gamma$ jointly, after which the function PIT-CDF can be evaluated for any new values of $\x$ and $\gamma$. Finally, Section~\ref{sec:reshape} describes how the learnt PIT-CDF itself defines a push-forward map (Equation~\ref{eq:recalibrated_PD}) that reshapes the densities so as to achieve approximate local calibration. Algorithm~\ref{alg:CalPIT} summarizes the details of the \calpit method.

\subsection{Estimating the PIT-CDF}
\label{sec:diagnostics}
We observe that the PIT-CDF in Equation~\ref{eq:r_alpha} is the regression (conditional mean) of a binary random variable $W^{\gamma} := \I(\PIT(Y;\X) < \gamma)$ on $\X$; that is, 
\begin{equation} \label{eq:PIT-CDF}
r^{\hat{f}}(\gamma; \x) = \mathbb{E} \left[ W^\gamma \ \middle| \ \x \right],
\end{equation} where the expectation $\mathbb{E}[\cdot \mid  \x]$ denotes an average with respect to the (unknown) target distribution $F(y|\x)$. The above expression indicates that we can estimate the PIT-CDF across the entire feature space $\mathcal{X}$, as well as for different quantiles $\gamma \in (0,1)$, via regression methods.

\calpit is implemented as follows: First, we augment the calibration data $\mathcal{D}$ by drawing multiple quantile values $\gamma_{i,1}, \ldots, \gamma_{i,K} \sim U(0,1)$ for each calibration data point ($i=1,\ldots,n$) and some chosen hyperparameter $K$. Next, we define the random variable $$W_{i,j} := \mathbb{I}(\PIT(Y_i; \mathbf{X}_i) \leq \gamma_{i,j}).$$ Finally, we train a suitable regression method using the augmented calibration sample $\mathcal{D}' = \{(\mathbf{X}_i, \gamma_{i,j}, W_{i,j})\}_{i,j}$ to predict $W_{i,j}$ with $(\mathbf{X}_i, \gamma_{i,j})$ as inputs, for $i=1,\ldots,n$ and $j=1,\ldots,K$. The computed regression function is an estimate of $\pr \left( \PIT(Y; \X) \leq \gamma \ \middle| \ \x \right)$. Since $r^{\hat{f}}(\gamma; \mathbf{x})$ is a non-decreasing function of $\gamma$, we typically choose monotonic neural networks \citep{Wehenkel2019UMNN} minimizing the binary cross-entropy loss \citep{BCELoss} as our regression method, especially for applications with complex and high-dimensional inputs of different modality. This loss function enforces $\pr \left( \PIT(Y; \X) \leq \gamma \ \middle| \ \x \right)$ to be well estimated \citep{dawid2014theory}.

\subsection{Reshaping Conditional Densities by Mapping Probabilities to Probabilities} \label{sec:reshape}
\calpit\ uses the estimated PIT-CDF to reshape the initial CDE $\widehat f$ into a new CDE $\widetilde f$ that is approximately locally consistent across the {\em entire} feature space. 
 
Our procedure for morphing one probability density into a new ``recalibrated'' density works as follows:
Consider a fixed evaluation point $\x$ and any $y_0 \in \mathcal{Y}$.  Let $\gamma := \hat{F}(y_0|\x)$.  If the regression is perfectly estimated (that is, $\widehat r^{\widehat f}=r^{\widehat f}$),  then, as long as both $F$ and $\widehat F$ are continuous and $\widehat F$ dominates $F$ (see Assumptions \ref{assump:continuity} and \ref{assump:dominates} in  \ref{sec:theory} for details), it holds that
$$	r^{\widehat f}(\gamma;\x) := \pr\left(\hat F(Y|\x) \leq  \gamma  \ \middle| \ \x \right) =  \pr\left(Y \leq y_0 \ \middle| \ \x \right)
	=F(y_0|\x).$$
In other words, the regression function $r^{\widehat f}$ changes the initial CDE so that the probability of observing the response variable $Y$ below $y_0$ is now indeed $F(y_0|\x)$ rather than $\hat{F}(y_0|\x)$.

It follows directly that for fixed $\hat F$,
$$r^{\hat f}\left(\hat F(y|\x);\x\right) =  \pr\left(\hat F(Y|\x) \leq  \hat F(y|\x)   \ \middle| \ \x \right) = \pr\left(Y \leq y \ \middle| \ \x \right)
	=F(y |\x)$$
The above result suggests that we can use the estimated regression, $\hat{r}^{\hat f}$, which is an approximation of the PIT-CDF, $r^{\hat f}$, to transform the original distribution $\hat F$ with density $\hat f$ into a  new  ``recalibrated'' conditional distribution $\widetilde F$ with density $\widetilde f$:

  \begin{Definition}[Recalibrated CDE]\label{def:recalibrated_PD}
  The recalibrated CDE of $Y$ given $\x$ is defined through the P-P map,
 \begin{equation}\widetilde F(y|\x):= \hat{r}^{\hat f}\left(\hat F(y|\x);\x\right),
 \label{eq:recalibrated_PD}\end{equation} 
 where $\hat{r}^{\hat f}$ is the regression estimator of the PIT-CDF (Equation~\ref{eq:r_alpha}).
\end{Definition}

If the PIT-CDF is well-estimated, then the new CDE will  achieve instance-wise calibration.  
The next theorem shows that, under some assumptions,
we can directly relate the quality of the recalibration (or how close the ``recalibrated'' distribution is to the target distribution in a mean-squared-error sense)  to the mean-squared-error of the regression estimator:
\begin{thm}[Performance of the recalibrated CDE]
\label{thm:performance_PD}
Under Assumptions  \ref{assump:continuity}, \ref{assump:dominates} and \ref{assump:bounded} (\ref{sec:theory}),
$$ \E \left[ \int \int \left(\widetilde F(y|\x)-F(y|\x) \right)^2  dP(y,\x)\right]=K \E \left[ \int \int \left(\widehat r^{\widehat f}(\gamma;\x)-r^{\widehat f}(\gamma;\x) \right)^2 d\gamma dP(\x) \right],$$ 
\end{thm}
\noindent The rate of convergence of $\widetilde F(y|\x)$ to the target distribution $F(y|\x)$  is given by Corollary~\ref{cor:rate}.

\begin{algorithm}[!ht]
	\caption{
		\calpit 
	}\label{alg:CalPIT}
		\algorithmicrequire \ {\footnotesize initial CDE $\hat f(y|\x)$ evaluated at $y\in G$;
		calibration set  $\mathcal{D}=\{(\X_1,Y_1),\ldots,(\X_n,Y_n)\}$; oversampling factor $K$; 
		  evaluation points $\mathcal{V} \subset \mathcal{X}$; 
		 nominal miscoverage level $\alpha$,  flag \texttt{HPD} (true if computing HPD sets)
	}\\
		\algorithmicensure \ {\footnotesize new distribution $\widetilde{F}(y|\x)$, \calpit interval $C(\x)$, new density estimate  $\widetilde{f}(y|\x)$, for all $\x \in \mathcal{V}$}\\
		
	\begin{algorithmic}[1]
	    \State \codecomment{Learn PIT-CDF from augmented and upsampled calibration data $\mathcal{D'}$}
		\State  Set $\mathcal{D'} \gets \emptyset$
	    \For{$i$ in $\{1,...,n\}$}
		    
		    \For{$j$ in $\{1,...,K\}$} 
		         \State Draw $\gamma_{i,j} \sim U(0,1)$
	 		    \State Compute $W_{i,j} \gets \I  \left(\PIT(Y_i; \X_i) \leq \gamma_{i,j} \right)$
	 		    \State Let $\mathcal{D'} \gets  \mathcal{D'}  \cup \{\left(\X_i,\gamma_{i,j}, W_{i,j} \right)\}$   
 \EndFor
	 	\EndFor  
			    \State Use  $\mathcal{D'}$ to learn $\hat r^{\hat f}(\gamma;\x) := \hat{\pr} \left( \PIT(Y;\x) \leq \gamma \ \middle| \ \x \right)$ via a regression of $W$ on $\X$ and $\gamma$, which is monotonic w.r.t. $\gamma$. 
			    \\
	\State 
        \codecomment{Map initial CDE into a new CDE by applying learnt PIT-CDF} 
		\For{$\x \in \mathcal{V}$ }
        \State \codecomment{Construct recalibrated CDE}
        
		\State Compute $\widehat{F}(y|\x) \gets \mathrm{cumsum}(\widehat f(y|\x))$ for $y \in G$
		  \State Let $\widetilde F(y|\x) \gets  \hat r^{\hat f}\left( \widehat F(y|\x) ;\x\right)$ for $y \in G$    
		  \State Apply interpolating (or smoothing) splines to obtain $\widetilde{F}(\cdot|\x)$ and $\widetilde{F}^{-1}(\cdot|\x)$ 
		 \State Differentiate $\widetilde F(y|\x)$ to obtain new distribution $\widetilde f(y|\x)$ for $y \in G$
  		\State Renormalize $\widetilde f(y|\x)$ according to \citet[Section 2.2]{izbicki2016nonparametric}
		  \State
		  \State \codecomment{ Construct Cal-\texttt{PIT} interval with conditional coverage $1-\alpha$} 		
		\State Compute $C(\x) \gets [\widetilde F^{-1}(0.5\alpha|\x);  \widetilde F^{-1}(1-0.5\alpha|\x) ]$.
  		\If{\texttt{HPD}}
  		
  		\State Obtain HPD sets
  		 $C(\x)=\{y: \widetilde f(y|\x)\geq \widetilde t_{\x,\alpha}\}$,
	  where $\widetilde t_{\x,\alpha}$ is such that 
	  $\int_{y \in C_\alpha(\x)} \widetilde f(y|\x)dy=1-\alpha $
		\EndIf
		\EndFor
		\State \textbf{return} $\widetilde{F}(y|\x)$, $C(\x)$, $\widetilde f(y|\x)$,  for all $\x \in \mathcal{V}$ 	
	\end{algorithmic}
\end{algorithm}

 Algorithm~1 details the \calpit procedure for computing the PIT-CDF from calibration data, and for constructing recalibrated  CDEs and prediction intervals. In practice, for each $\x$ of interest, we first evaluate  $\widetilde F(y|\x)$ across a grid $G$ of $y$-values, and then use linear or spline-based interpolation scheme to calculate the derivatives to finally obtain $\widetilde f(y|\x)$, our estimate of the recalibrated CDE at $\x$.

\begin{Remark}\label{remark:support}
If the initial model is good, then $r$ is easy to estimate; for instance, $\widehat f=f$ implies a constant function $r^{\hat f}(\gamma;\x)=\gamma$. However, $\widehat f$ needs to have support on the entire range of the target variable $y$ across the feature space $\mathcal{X}$. Depending on the application, a viable initial model could, for example, be an estimate of the marginal distribution $f(y)$, a uniform distribution with finite support (as in Experiment 2 of \ref{app:example_prediction_sets}, Example 3), an initial fit of the density with a Gaussian distribution (as in the TC application in Section~\ref{sec:example_TCs}), or a nonparametric density estimate (as in Experiment 1 of \ref{app:example_prediction_sets}, Example 3). In the \pz application in Section~\ref{sec:photo-z}, we use a weighted sum of the marginal distribution $f(y)$ and a Gaussian model for $f(y|\x)$. The Gaussian model was obtained from a widely popular \pz method (\texttt{GPz}; \citealt{Almosallam2016GPz}); the marginal distribution was then added to expand the support of the fitted Gaussian distribution. 
\end{Remark}

\begin{Remark}[CDEs and Prediction Sets]
 ~\label{remark:prediction_sets} 
As a by-product of conditional distributions,  one can derive various quantities of interest, such as moments, kurtosis, prediction intervals, or even more general prediction bands; such as Highest Predictive Density (HPD) regions $\left\{y: f(y|\x)> c \right\},$ where $f$ is the conditional density associated to $F$; see \ref{app:cal_hpd} for details on how to compute HPD regions. By construction, locally calibrated CDEs yield prediction bands with approximately correct {\em conditional} coverage. That is, suppose that $C_{\alpha}(\X)$ is a $(1-\alpha)/100\%$   prediction band derived from the CDF $\widehat F$. Local calibration of $\widehat F$ then implies that the prediction bands $C_{\alpha}(\X)$ 
  have approximate nominal coverage
\begin{equation} \label{eq:cond_validity}
  \mathbb{P}(Y\in C_\alpha(\X) |\X = \x) = 1 -\alpha, 
\end{equation}
for {\em every} instance  $\x \in \mathcal{X}$.
On the other hand, it is difficult to convert prediction  bands  and quantile estimates to entire  CDEs without additional assumptions. That is, calibrated CDEs imply calibrated prediction bands but not vice versa. For example, Theorem \ref{thm:consistency} in \ref{sec:theory} shows that a \calpit \ prediction interval at $\x$, defined as
\begin{equation} \label{eq:CalPIT_int}
C_\alpha(\x) :=\left[\widetilde F^{-1}(0.5\alpha|\x), \ \widetilde F^{-1}(1-0.5\alpha|\x)\right], 
\end{equation}
 achieves asymptotic conditional coverage, even if the initial CDE $\hat{f}(y|\x)$ is not consistent.
\end{Remark}

\section{Synthetic Examples} \label{sec:toy_examples}

\subsection{Example 1: Diagnostics and Reshaping of CDEs via P-P maps}\label{sec:example_sinh_arc_sinh}

This example illustrates the LADaR framework with \calpit: We start with an initial model $f_0(y|\x)$ (in this case, a Gaussian density with correct mean and fixed variance). Then, via a PIT-CDF regression (Equation~\ref{eq:PIT-CDF}), we learn the local diagnostics which can be visualized via P-P plots similar to Figure~\ref{fig:PPplot_interpret}. Finally, we reshape the initial densities to better fit the calibration data by applying the same learned P-P map  (that is, the PIT-CDF transformation) to the initial densities (Equation~\ref{eq:recalibrated_PD}).

As an illustration, we create a ``skewed'' data setting. The data are drawn from the family of sinh-arcsinh normal distributions \citep{jones2009sinh,jones2019sinhnormal}, where the {\em skewed} data follow
$$ Y_A | X \sim \textrm{sinh-arcsinh}(\mu=X, \sigma=2-|X|, \gamma=X, \tau=1),$$

We start with an initial Gaussian model given by
$$ Y|X \sim \mathcal{N}(\mu=X, \sigma=2),$$
and we learn the PIT-CDF function $\widehat r^{f_0}(\gamma;\x)$ from a calibration set of $n=10000$ pairs of $(X,Y)$.

The top panel of figure~\ref{fig:ex1_skewed} shows ``Local Amortized Diagnostics'' for the skewed setting: The first row graphs a {\em local discrepancy score} (LDS) across the feature space (see \citealt{Kodra2023} for an example use of the global analog), where the LDS is defined as
\begin{equation} 
D(\x):=\frac{1}{|G|}\sum_{\gamma \in G} (\widehat r^{f_0}(\gamma;\x) -\gamma)^2, \label{eq:discrepancy_score}
\end{equation} 
for a set $G \subset [0,1]$ of $\gamma$ values. The LDS is a one-number summary that estimates the amount of discrepancy between the initial model and the true density in terms of coverage: a large value of $D(\x)$ indicates that $f_0$ is miscalibrated at the evaluation point $\x$. The PIT-CDF function $\hat{r}^{f_0}$ then provides more detailed information on {\em how} the initial model $f_0(y|\x)$ might deviate from the true density $f(y|\x)$ at $\x$, 
as illustrated by the shape of the P-P plots in the second row. Top panel II (``Reshaping of Densities'')  shows examples of morphing the initial density $f_0$ (\textcolor{blue}{blue}) into an approximation $\widetilde{f}$ (\textcolor{red}{red} ) of the final density defined by Equation~\ref{eq:recalibrated_PD}. For illustrative purposes, we show intermediate curves $s\widetilde{f} + (1-s)f_0$ for a few different values $s \in [0,1]$.

Finally, because we know the true data-generating distribution $F$, we can directly assess the quality of the reshaped densities $\widetilde{f}$ by first generating MC samples from the true distribution at each evaluation point $\x$, and then computing a local version of the continuous rank probability score {\em} (CRPS). More specifically: CRPS is a proper scoring rule commonly used to evaluate probabilistic predictions \citep{matheson1976scoring}. The local CRPS loss at a point $(\x,y)$ is typically defined as 
\begin{equation}
L_{\mathrm{CRPS}}(\widetilde{f};\x, y) = \int_{-\infty}^{\infty} \left( \widetilde{F}(t|\x) - \I(y \leq t) \right)^2 dt,
\end{equation}
which checks whether a single draw $y \sim F(Y|\x)$ (from the unknown true distribution $F$) is consistent with the estimated distribution $\widetilde{F}(y|\x)$.
However, for our synthetic examples, we can generate an entire MC sample $Y_1, \ldots, Y_B \sim F(y|\x)$ (from the known true distribution $F$) at {\em any} fixed evaluation point $\x$ for some chosen large value $B$. We then define the local Monte Carlo CRPS (MC-CRPS) loss at fixed $\x$ as

\begin{equation}
    \label{eq::localMCCRPS}
 L_{\mathrm{MC-CRPS}}(\widetilde{f};\x, f) =  \int \left(\widetilde F(t|\x) - \frac{1}{B} \sum_{b=1}^B I(Y_b<t) \right)^2 dt, 
\end{equation}
For large $B$, Equation~\ref{eq::localMCCRPS} is close to zero when $\widetilde F(\cdot|\x)$  is a good estimate of $F(\cdot|\x)$. Furthermore,  Equation \ref{eq::localMCCRPS} is, up to a constant that does not depend on $\widetilde F$,
approximately the same as $\E\left[L_{\mathrm{CRPS}}(\widetilde{f};\x, Y) \mid \x \right]$, the conditional mean of the CRPS loss given $\X=\x$ (see \ref{appendix:local_CRPS} for more details). The bottom panel of Figure~\ref{fig:ex1_skewed} shows the local MC-CRPS results before and after applying \calpit\ for the ``skewed'' setting. The corresponding results for a ``kurtotic'' setting can be found in~\ref{app:example_kurtotic}.

\begin{figure}[H]
\hspace*{-0.0cm}
\centering
\includegraphics[trim={0 3cm 0 0},clip=true,width=0.9\columnwidth]{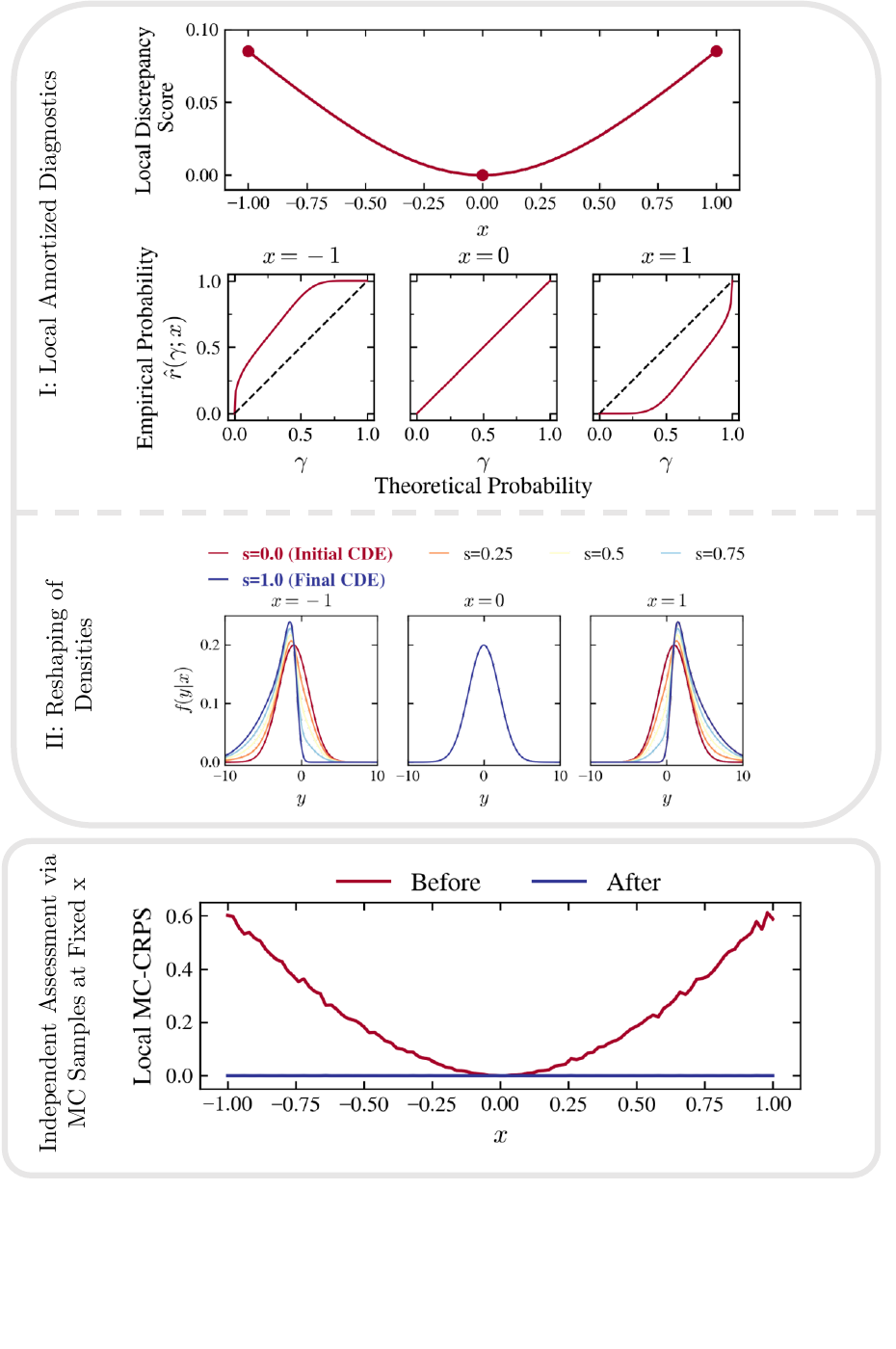}
		\caption{\footnotesize {\bf Illustration of LADaR framework: Example 1 skewed data.} Initial CDE is Gaussian, but the true distribution is skewed. {\em Top panel (I):} Local discrepancy score across the input space (first row) and examples of diagnostic P-P plots (second row). \calpit identifies that the model is {\em positively/negative biased} relative to calibration data at $X =-1$ / $X =1$ but well-estimated at $X =0$. The diagnostics define a family of P-P maps for reshaping the initial densities to fit the calibration data across the feature space. {\em Top panel (II):} Continuous morphing of densities via \calpit, illustrated at the three evaluation points, from the initial Gaussian distributions (\textcolor{red}{\bf red}; $s=0$) to the final distributions (\textcolor{blue}{\bf blue}; $s=1$). For illustrative purposes, we have included intermediate values of $s$ to show the morphing of distributions. {\em Bottom panel:} Independent assessment of final results by computing a local Monte Carlo version of the continuous ranked probability score (MC-CRPS) at fixed $x$ before and after \calpit.}
  \label{fig:ex1_skewed}
\end{figure}

\subsection{Example 2: Probabilistic Nowcasting with High-Dimensional Sequence Data as Inputs} \label{sec:example_TCs}

Our next synthetic example is motivated by short-term forecasting of the intensities of tropical cyclones (TC) with high-resolution satellite images. This application is challenging both because of the high-dimensional nature of spatio-temporal satellite data and because the intensities are auto-correlated in time. Figure~\ref{fig:GOES}, right, shows an example of a 24-hour sequence $\S_{<t}$ of consecutive radial profiles (one-dimensional functions) extracted from Geostationary Operational Environmental Satellite (GOES) infrared imagery \citep{janowiak2020NOAA}.

Infrared (IR) imagery, as observed by GOES, measures the cloud top temperature, which is a proxy for the strength of convection (the key component of the mechanism through which TCs extract energy from the ocean). Hence, each computed sequence $\S_{<t}$ can be seen as a summary of the spatio-temporal evolution of the convective structure of the TC leading up to time $t$, where patterns in $\S_{<t}$ signaling strengthening/weakening convection are predictors of intensifying/weakening storms; that is, they predict changes in the intensities of the TC, $I_\tau$, for $\tau \geq t$.

\begin{figure*}[htb]	
		\begin{center}
			\includegraphics[width=0.9\columnwidth]{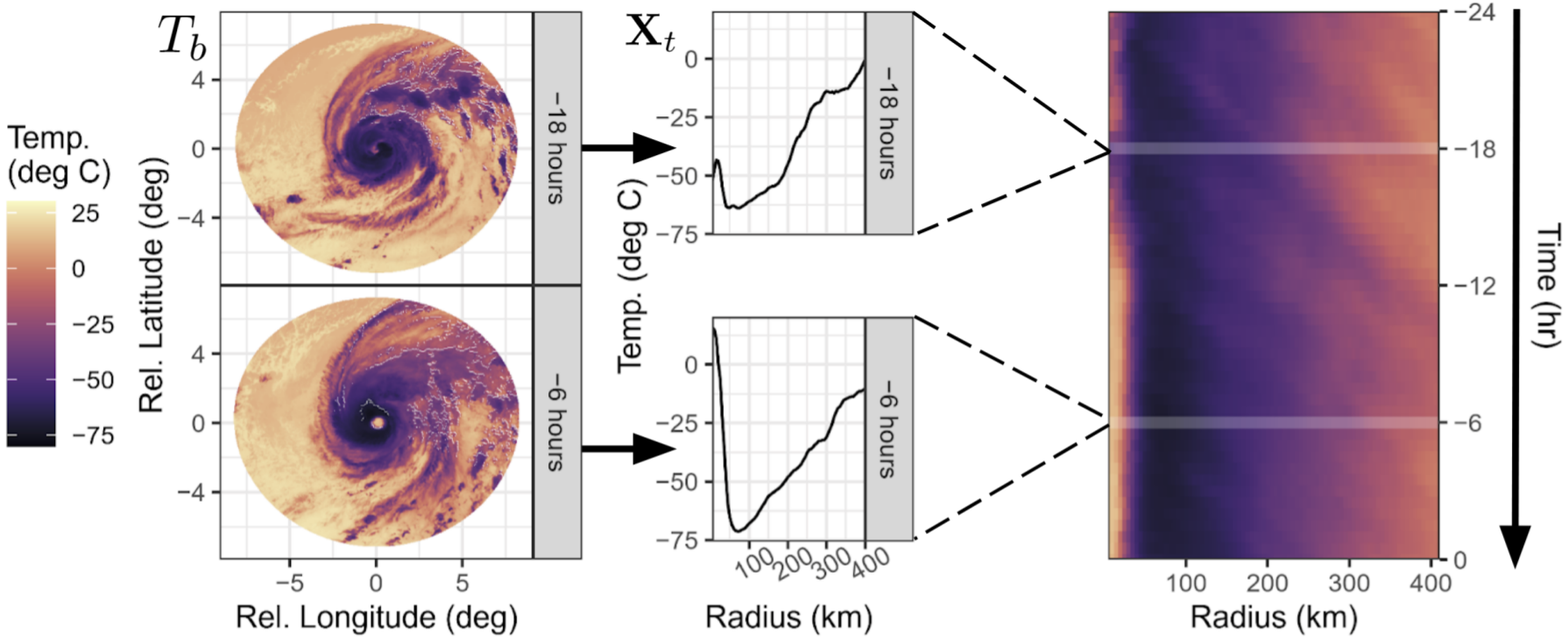}
		\end{center}	

		\caption{ 
			{\footnotesize {\bf TC satellite images.} \textit{Left:} A sequence of TC-centered cloud-top temperature images from GOES. \textit{Center:} We represent each GOES image with a radial profile of azimuthally-averaged cloud-top temperatures. \textit{Right:} The 24-hour sequence of consecutive radial profiles, sampled every 30 minutes, defines a structural trajectory $\S_{<t}$ or Hovm\"{o}ller diagram. Figure adapted from \citet{mcneely2022TCs}.}
		}\label{fig:GOES}	
\end{figure*}

As a proof-of-concept of our LADaR framework, we create a synthetic example with the same format as actual TC data. The details are described in supplementary material~\href{https://lee-group-cmu.github.io/cal-pit-paper/supplementary_material.pdf}{S3}\footnote[1]{Supplementary Materials:\\ \url{https://lee-group-cmu.github.io/cal-pit-paper/supplementary_material.pdf}}. Figure~\ref{fig:TC_data} shows an example of a simulated storm. On the left, we have a toy Hovm{\"o}ller diagram of the evolution of the ``convective structure'' $\{(\X_t)\}_{t \geq 0}$, with each row representing the radial profile $\X_t \in \mathbb{R}^{120}$ of temperature as a function of radial distance from the storm center; time evolution is top-down in hours.  On the right, we have $\{Y_t\}_{t \geq 0}$, the simulated ``TC intensities'' at corresponding times $t$. The trajectory $\S_{<t}:=(\X_{t-47}, \X_{t-46}, \ldots,\X_{t})$ represents the 24-hour history of the convective structure (48 radial profiles). We simulate 8000 ``storms'' according to a fitted TC length distribution. Sequence data $\{(\S_{<t}, Y_t)\}$ from the same storm are shifted by 30 minutes; therefore, they are strongly correlated. Sequence data from different storms, on the other hand, are independent.
\begin{figure}[H]
\begin{minipage}[ht]{0.6\columnwidth}
		\begin{center}
			\includegraphics[width=1\columnwidth]{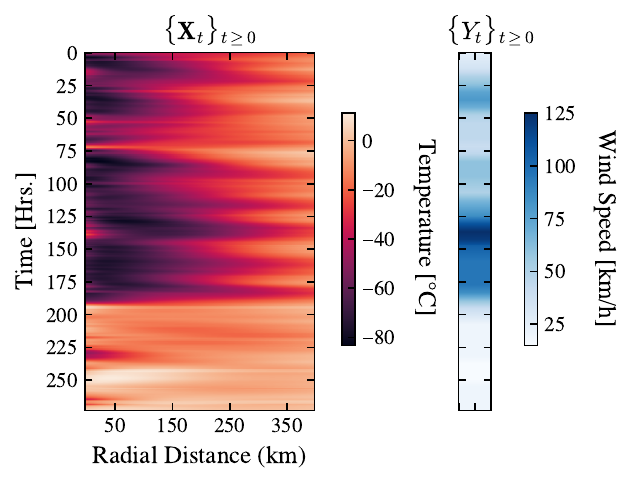}
		\end{center}
	\end{minipage}
	\begin{minipage}[ht]{0.38\linewidth} 
		\baselineskip=13pt
		\caption{
        \footnotesize {\bf Synthetic data in Example 2.}  
       Simulated radial profiles $\{\X_t\}_{t \geq 0}$ and intensities  $\{Y_t\}_{t \geq 0}$ for an example TC.
		\textit{Left:}   Each row represents the radial profile  $\X_t$ of temperature as a function of radial distance from the storm center at time $t$.    Our predictors are 48-hour overlapping sequences $\{\S_t\}_{t \geq 0}$ with data from the same ``storm'' being highly dependent.
		\textit{Right:} The target response,  here shown as a time series $\{Y_t\}_{t \geq 0}$ of simulated TC intensities.
        }
		\label{fig:TC_data}
	\end{minipage}
	
\end{figure}
Our goal is to ``nowcast'' the conditional distribution $Y_t|\S_{<t}$, where $Y_t$ is the intensity at time $t$. Here we illustrate how \calpit can diagnose and improve an initial convolutional mixture density network (ConvMDN) model. In our example, we perform training, calibration, and testing on different simulated ``storms'': First, we fit an initial CDE (ConvMDN; \citealt{disanto2018cmdn}), which estimates $f(y|\s)$ as a unimodal Gaussian, based on a train set with 8000 points, $\{(\S_{<t},Y_t)\}$ (see supplementary material~\href{https://lee-group-cmu.github.io/cal-pit-paper/supplementary_material.pdf}{S3} for details). Next, we apply \calpit  to learn $\hat r^{\hat f}(\gamma;\s)$  using 8000 calibration points.  (Note that the data within the same storm are highly dependent; hence, the effective train or calibration sample sizes are much smaller than the nominal values.) 
 Because we have access to the data-generating distribution, we can assess the performance of CDEs before and after reshaping densities by MC samples at 4,000 test points.

Figure~\ref{fig:tc_phase_space}	summarizes the results. With the LADaR framework (top panel), we are able to identify regions in a high-dimensional space of sequence data where our initial CDE of $Y_t|\S_{<t}$ is a poor fit. In the upper left panel, each point corresponds to a 24-hour structural trajectory $\S_{<t}$ or a sequence of radial profiles visualized in a reduced dimensionality space using principal component analysis (PCA); the points are color-coded by the local discrepancy score (LDS) between the initial model and the true distribution of the calibration data according to \calpit. Three specific examples of input sequences are also shown. After applying the estimated P-P map via \calpit to all CDEs, we obtain near instance-wise calibration according to an independent MC assessment (bottom panel).

\begin{figure}[H]	
		\begin{center}
			\includegraphics[trim={0 1.5cm 0 0},clip=true,width=0.78\columnwidth]{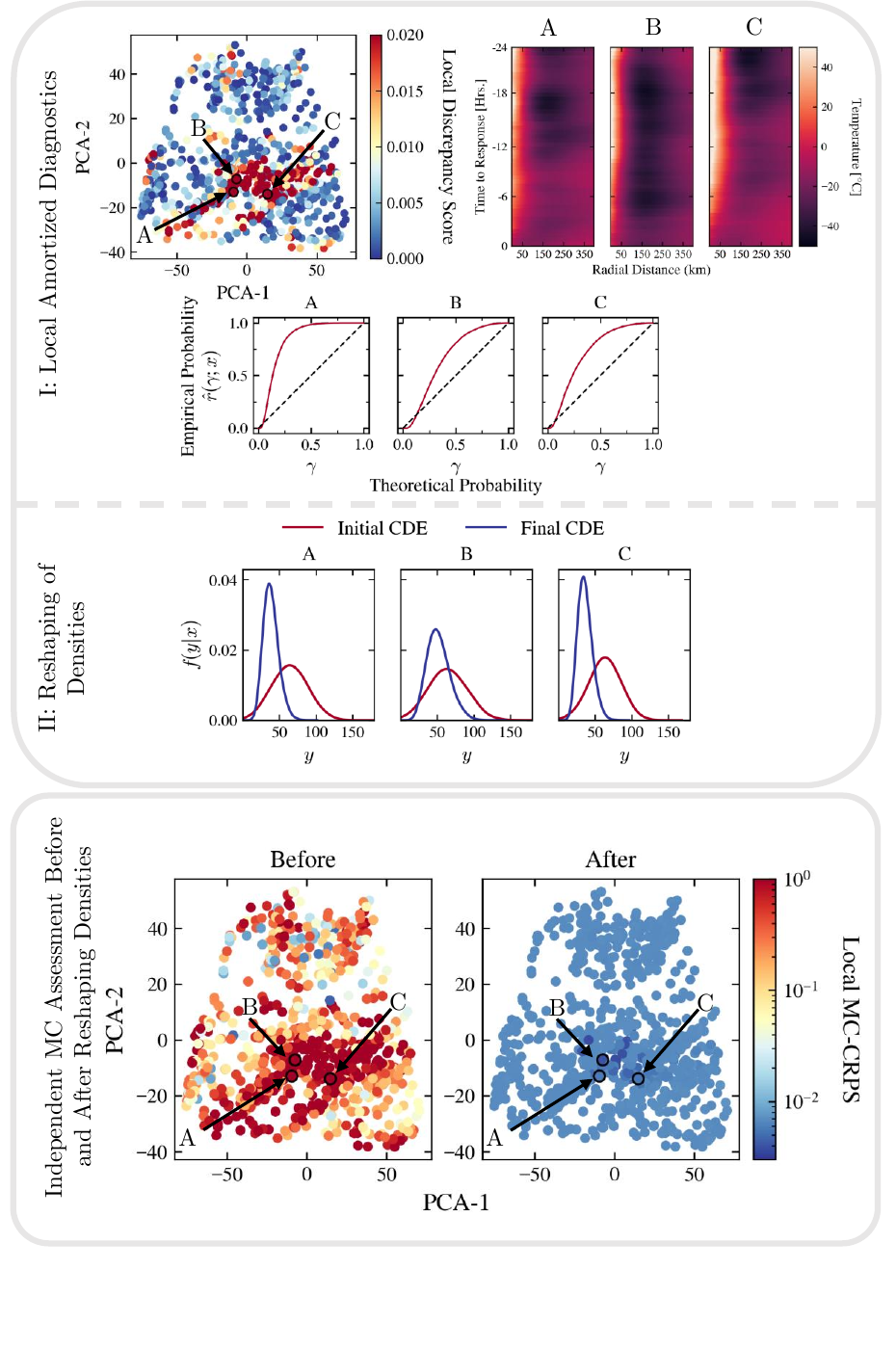}
		\end{center}	
\vspace{-0.25in}
		\caption{ 
			{\footnotesize {\bf Example 2: Probabilistic nowcasting with high-dimensional sequence data as inputs.}
            {\em Top panel I: Local Amortized Diagnostics.} First row: Two-dimensional PCA map of sequence data. One point in the map represents a 24-h structural trajectory $\S_{<t}$ or sequence of radial profiles; the points are color-coded by the local discrepancy score (LDS) between the initial model and calibration data according to \calpit. Points A--C represent three examples of inputs $\S_{<t}$ where the initial model appears to perform the worst (i.e., high LDS).  
            Second row: P-P plots help reveal the nature of the discrepancy; the initial model appears positively biased and overdispersed at the three locations. {\em Top panel II: Reshaping of densities.}  The density of $Y_t|\S_{<t}$  before (\textcolor{red}{\bf red}) and after (\textcolor{blue}{\bf blue}) applying the P-P map. {\em Bottom panel: Independent MC assessment.} For synthetic data, we can compute the continuous ranked probability score (CRPS) locally for simulated MC samples at fixed $\S_{<t}$. The local Monte Carlo CRPS scores are shown before (left) and after (right) reshaping the densities. After applying the P-P map, the CDEs are well-calibrated for all inputs $\S_{<t}$.}}
		\label{fig:tc_phase_space}	
\end{figure}

\subsection{Example 3: Prediction Sets} 

The novelty of our method lies in the fact that we can construct full CDEs with approximate instance-wise coverage. Nevertheless, as a by-product, we can also construct prediction sets with approximate conditional coverage. Given the wealth of previous literature on prediction sets, we have added an additional synthetic example in \ref{app:example_prediction_sets} to demonstrate that prediction sets derived from \calpit are competitive with sets from conformal inference and quantile regression.

\section{Main Application: Reshaping CDEs of Galaxy Photometric Redshifts} \label{sec:photo-z}

Many astrophysical studies depend on knowing the distances to external galaxies. Geometric distances to galaxies are incredibly difficult to measure, so astrophysicists typically use the redshift of light emitted from a galaxy as a proxy for its distance, where the spectral energy distribution (the intensity of light as a function of wavelength) is shifted to longer (redder) wavelengths due to the cosmological expansion of space. Redshifts can be precisely measured using spectroscopy to identify spectral features that occur at known wavelengths, but obtaining spectroscopic redshifts is resource-intensive. A far more efficient approach is to estimate redshifts from imaging data (i.e., \pzs), but even with measurements at several wavelengths, imaging data produce a less precise localization of these features (and hence more uncertain \pzs) due to a much coarser wavelength binning of photons. In particular, upcoming multi-billion dollar imaging projects like the Rubin Observatory's Legacy Survey of Space and Time (LSST; \citealt{Ivezic2019LSST}), the Nancy Grace Roman Space Telescope~\citep{Akeson2019RomanSurvey}, and the Euclid Mission~\citep{Laureijs2011EuclidSurvey} will make key cosmological measurements using weak gravitational lensing (see, e.g., \citealt{MandelbaumWLReview} for an overview), a method that relies on well-calibrated \pzs of millions of galaxies. The demands on the accuracy of \pz CDEs for these projects are extremely stringent: discrepancies in the moments of redshift distributions for samples that are instrumental in measuring cosmological parameters must be less than approximately $0.1\%$ to prevent degradation of subsequent physical analyzes \citep{DESCSRD}.

However, calibrating \pz CDEs remains tricky because galaxies span a wide range of intrinsic properties and spectral energy distributions \citep{Conroy2013SPSReview}, which leads to different combinations of redshift and intrinsic spectral energy distribution producing nearly identical observed imaging data. This problem is further complicated by measurement errors and the coarseness of the spectral information available from imaging data. Thus, the estimation of \pzs is inherently probabilistic with often non-trivial (e.g., non-Gaussian or bimodal) distributions. These distributions cannot be accurately captured by point estimates or prediction sets and must be quantified using full predictive distributions  \citep{Benitez2000BPz,Mandelbaum2008PhotozPDF,Malz2022Photoz}, which \calpit is uniquely suited to estimate.

Most \pz estimation methods fall into two main classes: physics-inspired methods that find the combination of redshift and spectral energy distribution that best matches the data (e.g., \citealt{Arnouts1999Lephare, Brammer2008EAZY}), and (ii) data-driven methods that learn a non-linear mapping between the input imaging data and redshift (e.g., \citealt{Beck2016Photoz, Zhou2021DESIPhotoz, Dalmasso2020FlexcodePhotoz, Dey2022Photoz}). No class of method is clearly the best for all imaging data sets, with the physics-based methods typically performing better when training data are sparse and the data-driven methods typically doing better when training data densely sample parameter space. Previous studies have used global metrics to reshape probability distributions (e.g., \citealt{ Euclid2021Recalibration, Kodra2023}), including PIT-based recalibration schemes (see, for instance, \citealt[Section 3]{bordoloi2010photoz}). Regardless, no method guarantees correct \textit{local} calibration of uncertainty estimates, a more stringent requirement that is the focus of \calpit.

To showcase the effectiveness of our LADaR approach, we utilize the data set from \citet{Schmidt2020Photo-z}, which has been used as a reference for assessing \pz CDE prediction techniques. This data set was developed by assigning realistic spectral energy distributions to galaxies in a dark matter-only simulation \citep{derose2019buzzard} to mimic their appearance in LSST imaging data. The input features consist of logarithmic measurements of intensity of observed galaxy light (spatially-integrated across the image) in a given wavelength range (corresponding to a photometric filter) called apparent \textit{magnitudes} and the differences between them called \textit{colors}. Additionally, uncertainty estimates for these measurements were also provided. For the \citet{Schmidt2020Photo-z} data challenge, the participants were given an unbiased ``training set'' of $\sim$44,000 instances (galaxies) to which they applied 11 different physics-inspired and data-driven \pz approaches. The \pz methods were then evaluated on an unseen ``test set'' of $\sim$400,000 instances (galaxies). For this exercise, the training set was perfectly representative of the test set. \citet{Schmidt2020Photo-z} also evaluated the performance of a method that simply predicted the marginal distribution of redshifts in the training set (i.e., the same prediction for every galaxy in the data set), which they called \texttt{trainZ}. Although this naive estimate does not contain any meaningful information about the redshift of any individual galaxy, \citet{Schmidt2020Photo-z} demonstrated that it can perform well on many commonly used metrics that check for marginal calibration.

Reassuringly, \citet{Schmidt2020Photo-z} found that \texttt{trainZ} performed very poorly on the conditional density estimate (CDE) loss \citep{izbicki2017converting}, a metric of conditional coverage. The CDE loss is a proper scoring technique and the \textit{conditional} analog of the root-mean-square-error for probabilistic regression. Given an estimate $\widetilde{f}$ of $f$, it is defined as the $L^2$ distance between $\widetilde{f}$ and $f$,
\begin{equation}
    L(f,\widetilde{f}) = \int \int [f(y|\mathbf{x})-\widetilde{f}(y|\mathbf{x}))]^{2}\mathrm{d}y\mathrm{d}P(\mathbf{x}),
\end{equation}
where $\mathrm{d}P(\mathbf{x})$ is the marginal distribution of features $\mathbf{x}$. The CDE loss cannot be evaluated directly as it depends on the unknown true density $f(y|\mathbf{x})$, but it can be estimated up to a constant ($K_{f}$, dependent on $f(y|\mathbf{x})$) by
\begin{eqnarray*}
    \widehat{L}(f,\widetilde{f}) = \E_{\mathbf{x}}\left[\int \widetilde{f}(y|\mathbf{x})^2\mathrm{d}y \right] - 2\E_{\mathbf{x},y}\left[ \widetilde{f}(y|\mathbf{x})\right] + K_{f},
\end{eqnarray*}
which is sufficient for relative comparisons between methods.
The CDE loss was the only conditional metric identified and tested in the \citet{Schmidt2020Photo-z} challenge, so we use it as our main metric for the assessment of \calpit in the context of \pzs and note that \calpit is independent of the CDE loss.

For a fair comparison, we adopt the same training and test sets from \citet{Schmidt2020Photo-z} and use the former as our calibration set to learn the local PIT-CDF . Among the methods compared by \citet{Schmidt2020Photo-z}, we use the density estimates from \texttt{GPz} \citep{Almosallam2016GPz} as our initial model. \texttt{GPz} uses sparse Gaussian processes to estimate the CDEs. Although, \texttt{GPz} produces Gaussian density estimates, it is commonly recognized that \pz conditional densities can have non-Gaussian characteristics such as long tails or bimodalities. To expand the support of the initial distributions, we took a weighted sum of the marginal distribution of redshifts in the calibration set and the \texttt{GPz} outputs  with weights 0.1 and 0.9, respectively, as our initial CDEs. We used monotonic neural networks to learn the PIT-CDF from an input feature set of one galaxy magnitude and five colors along with their measurement uncertainties. We then use the same features to diagnose and reshape the initial densities. Finally, we assess the quality of our reshaped CDEs with the CDE loss.

Figure~\ref{fig:photo-z-local} showcases how \calpit is a powerful tool for diagnosing and reshaping \pz CDEs. The top row of panel I displays a subset of the test data points in two projections (left: $u-g$ color vs.\ $i$-band magnitude; right: $r-i$ color vs.\ $z-y$ color) of feature space with the points color-coded by the local discrepancy score.  Four individual galaxies are highlighted, and their diagnostic P-P plots are shown in the second row of panel I. The first P-P plot shows an instance where the initial model was good and no substantial reshaping is necessary. The second P-P plot shows an instance where the initial guess is overdispersed, whereas the third shows an instance where the initial guess was heavily biased. The last P-P plot demonstrates a case where the P-P plot has multiple steep sections, indicating that initial model failed to express a bimodal density. 

Panel II shows the initial CDE (red), the reshaped CDE (blue), and the true redshift (dotted black line and cross).  \calpit leverages the information contained in the diagnostics (i.e., the P-P plots from panel I) to reshape the initial CDEs and even recover bimodal CDEs from unimodal input CDEs (with the true redshift being in one of the modes). Figure~\ref{fig:photo-z-local}~(bottom row) provides a clear (though not statistically rigorous) demonstration that the CDEs from \calpit are indeed meaningful. Since we do not know the ``ground truth'' distributions for this data set, we have to rely on indirect methods to assess the quality. Specifically, we use the distribution of true redshifts of other galaxies with similar imaging data. We identify those counterparts by searching for other galaxies in the test set whose magnitudes and colors (rescaled by subtracting the mean and dividing by the standard deviation for each feature) lie within a Euclidean distance of 0.5 units of our selected galaxies. Figure~\ref{fig:photo-z-local}~(bottom row) shows their redshift distribution as weighted histograms, where the weights are inversely proportional to the euclidean distance to each neighbor, together with their predicted CDEs. When CDEs are unimodal, the nearest-neighbor histograms are also unimodal with similar widths. Even more impressively, when our inferred CDEs are bimodal, the nearest-neighbor histograms show matching bimodal distributions, indicating that not only did \calpit correctly find the mode with the true redshift, but also correctly identified the other redshift solution with similar imaging properties.

Finally, Table \ref{tab:photo_z_comparison} shows that \calpit achieves a lower CDE loss than any of the methods in the LSST-DESC Photo-$z$ data challenge \citep{Schmidt2020Photo-z}. The values of the CDE loss for all methods except \calpit come from \citet{Schmidt2020Photo-z}, whereas the value for \calpit was obtained by running our algorithm on the same train and test sets. As expected, there is a major improvement in the value of the CDE loss (from $-9.93$ to $-10.80$) from our input distribution (i.e., \texttt{GPz}) to our \calpit-reshaped distributions. Moreover, \calpit outperforms all other \pz methods tested by \citep{Schmidt2020Photo-z}, including \texttt{FlexZBoost} \citep{izbicki2017converting}, which was designed to minimize the CDE loss. Although the improvement over \texttt{FlexZBoost} is not dramatic, \calpit guarantees proper calibration, which \texttt{FlexZBoost} does not. Because \calpit outperforms state-of-the-art \pz prediction methods on independent metrics while ensuring proper calibration, it is perhaps the most promising method for meeting the exacting \pz requirements of next generation imaging surveys.

\begin{figure}[H]
	\centering
	\includegraphics[trim={0 2.8cm 0 0},clip=true,width=0.8\columnwidth]{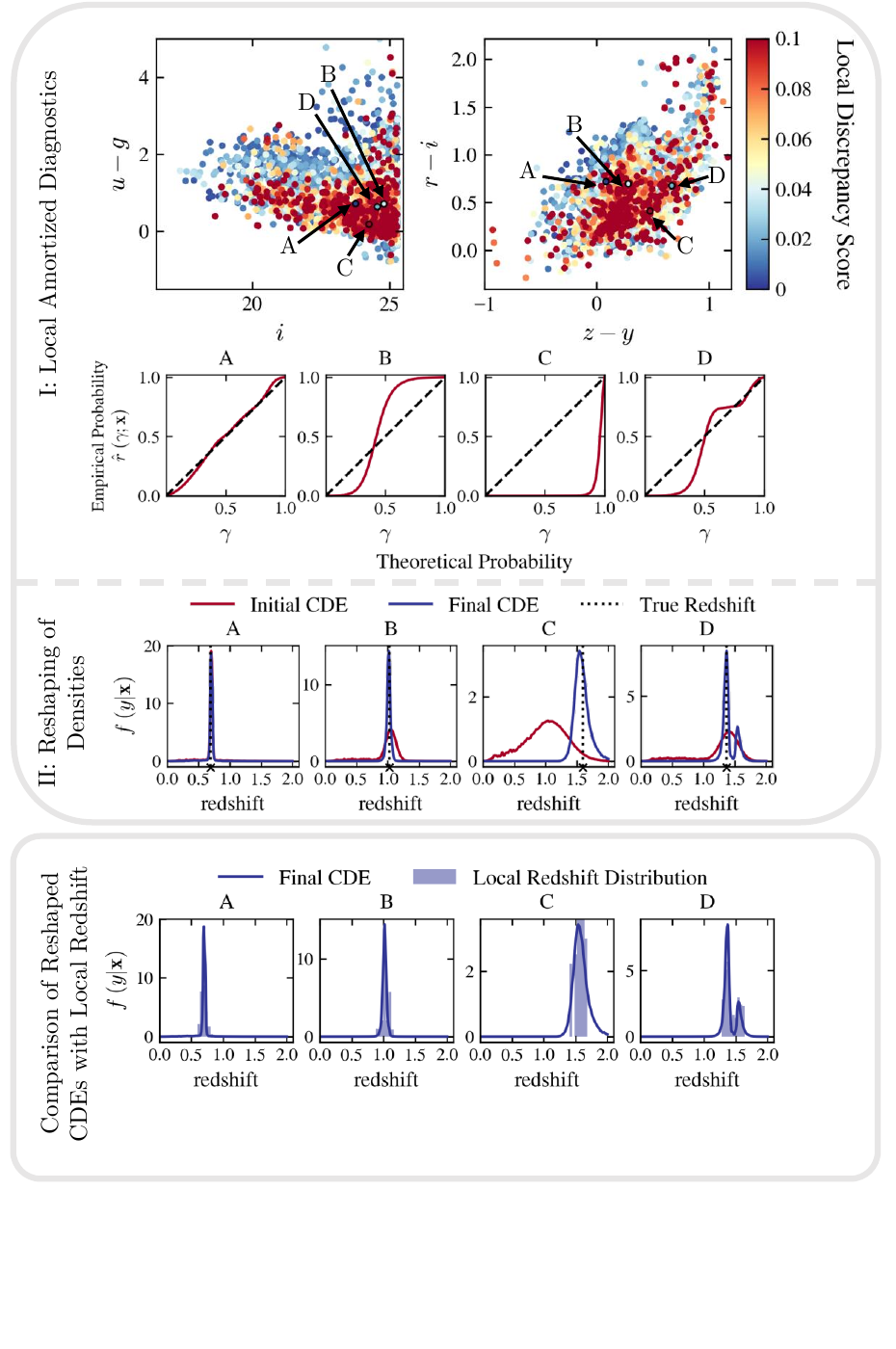}
	\caption{ \footnotesize {\bf Photo-$z$ application.} {\em Top panel I: Local amortized diagnostics.} First row: The local discrepancy score for the initial model shown in two projections of the feature space. The first figure shows galaxy $i$-band magnitude and $u-g$ color space, and the second figure shows $z-y$ color and $r-i$ color space. The points labeled A--D correspond to the four galaxies for which we show the diagnostics and reshaping. Second row: Diagnostic P-P plots of the initial model (modified \texttt{GPz} CDEs; \citealt{Almosallam2016GPz}) for four galaxies from the LSST-DESC Photo-$z$ Data Challenge \citep{Schmidt2020Photo-z} test set. {\em Top panel II: Reshaping of densities.} Photo-$z$ CDEs for the corresponding galaxies before (\textcolor{red}{\bf red}) and after (\textcolor{blue}{\bf blue}) reshaping the densities via \calpit; the true (spectroscopic) redshift is shown as a vertical dotted black line and a cross. \calpit can correct for bias and over-/under-dispersion.  Most impressively, it can recover accurate bimodal CDEs even if the initial estimate was unimodal. {\em Bottom row:} Comparison of the final reshaped CDEs (\textcolor{blue}{\bf blue} line)  with the local ``nearest-neighbor'' distribution (\textcolor{blue}{\bf blue} shaded histogram) of true redshifts of other galaxies with similar imaging properties. \calpit accurately approximates the local redshift distribution for unimodal and multimodal redshift distributions. Further, the inferred CDEs are bimodal only when the histograms are bimodal.}
	\label{fig:photo-z-local}
\end{figure}

\begin{table}[H]
    \centering
    \sisetup{detect-weight,
         mode=text, 
         table-format=2.0
        }
    \caption{\footnotesize{Comparison of the CDE loss values for \calpit and the methods benchmarked in the LSST-DESC Photo-$z$ Data Challenge~\citep{Schmidt2020Photo-z}. In terms of the CDE loss, \calpit performs better than all of the other methods tested, including \texttt{FlexZBoost}, which is specifically optimized to minimize the CDE loss.}}
    \label{tab:photo_z_comparison}
    \begin{tabular}{lS} 
    \hline
    Photo-$z$ Algorithm & \mcc{CDE Loss} \\
    \hline
    \texttt{ANNz2} \citep{Sadeh2016ANNZ2} & -6.88 \\
    \texttt{BPZ} \citep{Benitez2000BPz} & -7.82 \\
    \texttt{Delight} \citep{Leistedt2017Delight} & -8.33 \\
    \texttt{EAZY} \citep{Brammer2008EAZY} & -7.07 \\
    \texttt{\textbf{FlexZBoost}} \citep{izbicki2017converting} & \B -10.60 \\
    \texttt{GPz} \citep{Almosallam2016GPz} & -9.93 \\
    \texttt{LePhare} \citep{Arnouts1999Lephare} & -1.66 \\
    \texttt{METAPhoR} \citep{Cavuoti2017Metaphor} & -6.28 \\
    \texttt{CMNN} \citep{Graham2018CMNN}& -10.43 \\
    \texttt{SkyNet} \citep{Graff2014Skynet} & -7.89 \\
    \texttt{TPZ} \citep{CarascoKind2013TPZ} & -9.55 \\
    \texttt{trainZ} \citep{Schmidt2020Photo-z} & -0.83 \\
    \hline
    \textbf{\calpit} & \B -10.80 \\
    
    \hline
    \end{tabular}
   
\end{table}

\section{Discussion} \label{sec:discussion}
There has been a growing interest in conditional density and generative models (see \citealt{snowmassUQ} and references therein) --- however, there are few tools for assessing whether these methods yield trustworthy instance-wise UQ. 

Our proposed solution, LADaR with \calpit, draws on the success of high-capacity predictive algorithms, such as deep neural networks, to recalibrate CDEs in complex data settings with interpretable results and a minimum of assumptions.

 \calpit first assesses whether an initial conditional density model $\widehat F(\cdot|\x)$ is well calibrated for all inputs $\x$ with respect to calibration data,  and then  provides a mechanism for morphing the initial densities toward the distribution $F(\cdot|\x)$ of the reference data.  Any transformation is valid as long as both $\widehat{F}(\cdot | \x)$ and $F(\cdot | \x)$ are continuous functions and $\widehat{F}(\cdot | \x)$ dominates $F(\cdot | \x)$—that is, $\widehat{F}$ assigns positive probability to any region where $F$ does. Under these conditions (see \ref{sec:theory} for details), the recalibrated distribution is well defined, and the conditional PIT fully characterizes the conditional CDF of the target variable. This flexibility explains why a unimodal distribution can be transformed into a bimodal one, as seen in the photo-$z$ example. Our method does not impose shape constraints on the recalibrated density. \calpit also does not require exchangeability. Instead, it only requires stationarity (to ensure that the regression function remains stable over time) and a form of weak dependence (to allow the regression method to effectively learn from new data); hence, the method can be applied to (stationary) probabilistic time series forecasting. Individually calibrated CDEs automatically return conditionally calibrated prediction sets. However, \calpit works under the assumption that $Y$ is continuous and does not apply to classification tasks (unlike calibration schemes in, e.g., \citealt{Kull2019Conditional,Wald2017Conditional}).

 Although we focus on prediction problems, our approach also applies to Bayesian inference, where the goal is to estimate intractable posterior distributions $F(\theta|\x)$. This includes Simulation-Based Inference (SBI; \citealt{cranmer2020sbi}), which approximates posteriors using simulations instead of explicit likelihoods \citep{beaumont2002approximate,papamakarios2016fast,lueckmann2017flexible,greenberg2019automatic,izbicki2019abc}. \calpit can assess and recalibrate such estimates $\widehat{F}(\theta|\x)$—whether obtained via MCMC or neural methods—toward the true posterior. For implicit models like MCMC, 
 for a fixed $\x \in \mathcal{X}$ and $\theta \in \Theta$, we draw $\theta_1,\ldots,\theta_L \sim \widehat{F}(\cdot|\x)$ and approximate $\PIT(\theta; \x)$ using $L^{-1} \sum_{i=1}^L \I(\theta_i\leq \theta)$. Unlike simulation-based calibration (SBC; \citealt{talts2018validating}), which focuses on marginal validity, \calpit enables instance-wise recalibration and reveals local failure modes. Recent methods \citep{linhart2024c2st,torres2024model,wehenkel2024addressing} also offer local diagnostics or data-driven calibration, but \calpit uniquely combines feature-space interpretability with an amortized probability-probability map to correct individual CDEs. 
 
Finally, \calpit can potentially be extended to multivariate output vectors $\Y$ by the decomposition $f(\y|\x)=\prod_{i} f(y_i|\x,\y_{<i})$; thus performing \calpit corrections on autoregressive components of the conditional distribution. This is a particularly promising direction for deep autoregressive generative models \citep{van2016conditional,van2016pixel,vaswani2017Transformers,AutoregressiveDiffusion}. We are currently investigating whether \calpit can improve structural forecasts for short-term tropical cyclone intensity guidance \citep{mcneely2023structural}. We are also using a LADaR approach to quantify the {\em added} value in leveraging ML-methods and GOES imagery in TC track forecasting relative Numerical Weather Prediction operational forecasts.  See recent work by \citet{linhart2022validation} for a multivariate extension of \calpit specific to normalizing flows. Other open problems include fast sampling from recalibrated conditional distributions to generate ensemble forecasts in real time.

%%%%%%%%%%%%%%%%%%%%%%%%%%%%%%%%%%%%%%%%%%%%%%
%% Support information, if any,             %%
%% should be provided in the                %%
%% Acknowledgements section.                %%
%%%%%%%%%%%%%%%%%%%%%%%%%%%%%%%%%%%%%%%%%%%%%%
\ack
The authors would like to thank Tria McNeely for helpful discussions and for preparing the tropical cyclone data that were used to fit the TC-inspired model. ABL thanks Jing Lei and Larry Wasserman for valuable comments on the P-P map. BD is a postdoctoral fellow at the University of Toronto in the Eric and Wendy Schmidt AI in Science Postdoctoral Fellowship Program, a program of Schmidt Sciences. This research used resources of the National Energy Research Scientific Computing Center, a DOE Office of Science User Facility supported by the Office of Science of the U.S.\ Department of Energy under Contract No.\ DE-AC02-05CH11231 using NERSC award HEP-ERCAP0022859. This work is supported in part by NSF DMS-2053804, NSF PHY-2020295, and the C3.ai Digital Transformation Institute. The efforts of BD and JAN were supported by grant DE-SC0007914 from the U.S.\ Department of Energy Office of Science, Office of High Energy Physics. BD, BHA, and JAN acknowledge the support of the National Science Foundation under Grant No.\ AST-2009251. Any opinions, findings, conclusions or recommendations expressed in this material are those of the author(s) and do not necessarily reflect the views of the National Science Foundation. RI is grateful for the financial support of CNPq (422705/2021-7 and 305065/2023-8) and FAPESP (2023/07068-1).

%%%%%%%%%%%%%%%%%%%%%%%%%%%%%%%%%%%%%%%%%%%%%%
%% Single Appendix:                         %%
%%%%%%%%%%%%%%%%%%%%%%%%%%%%%%%%%%%%%%%%%%%%%%
%\begin{appendix}
%\section*{???}%% if no title is needed, leave empty \section*{}.
%\end{appendix}
%%%%%%%%%%%%%%%%%%%%%%%%%%%%%%%%%%%%%%%%%%%%%%
%% Multiple Appendixes:                     %%
%%%%%%%%%%%%%%%%%%%%%%%%%%%%%%%%%%%%%%%%%%%%%%
\appendix
\input{appendix}
% \unappendix

%\section{???}
%
%\section{???}
%
% \clearpage

\bibliographystyle{psj} % Style BST file
\newcommand{\newblock}{}
\bibliography{bibliography}       % Bibliography file (usually '*.bib')

%\newpage
%\input{supplement}

\end{document}

%% file: appendix.tex
\providecommand{\upGamma}{\Gamma}
\providecommand{\uppi}{\pi}

\clearpage

\section{Example 1: Synthetic Example (Kurtotic Setting)} \label{app:example_kurtotic}

Figure~\ref{fig:ex1_kurtotic}  presents the LADaR approach and the results for the ``kurtotic'' setting in Example 1. The data are drawn from the sinh-arcsinh normal distribution and follow
$ Y_B | X \sim \textrm{sinh-arcsinh}(\mu=X, \sigma=2, \gamma=0, \tau=1-X/4).$
The initial model is Gaussian given by
$ Y|X \sim \mathcal{N}(\mu=X, \sigma=2),$
and we learn the PIT-CDF function $\widehat r^{f_0}(\gamma;\x)$ from a calibration set of $n=10000$ pairs of $(X,Y)$.

 \begin{figure}[H]
		\hspace*{-0.0cm}
		\centering
	
        \includegraphics[trim={0 3cm 0 0},clip=true,width=0.65\columnwidth]{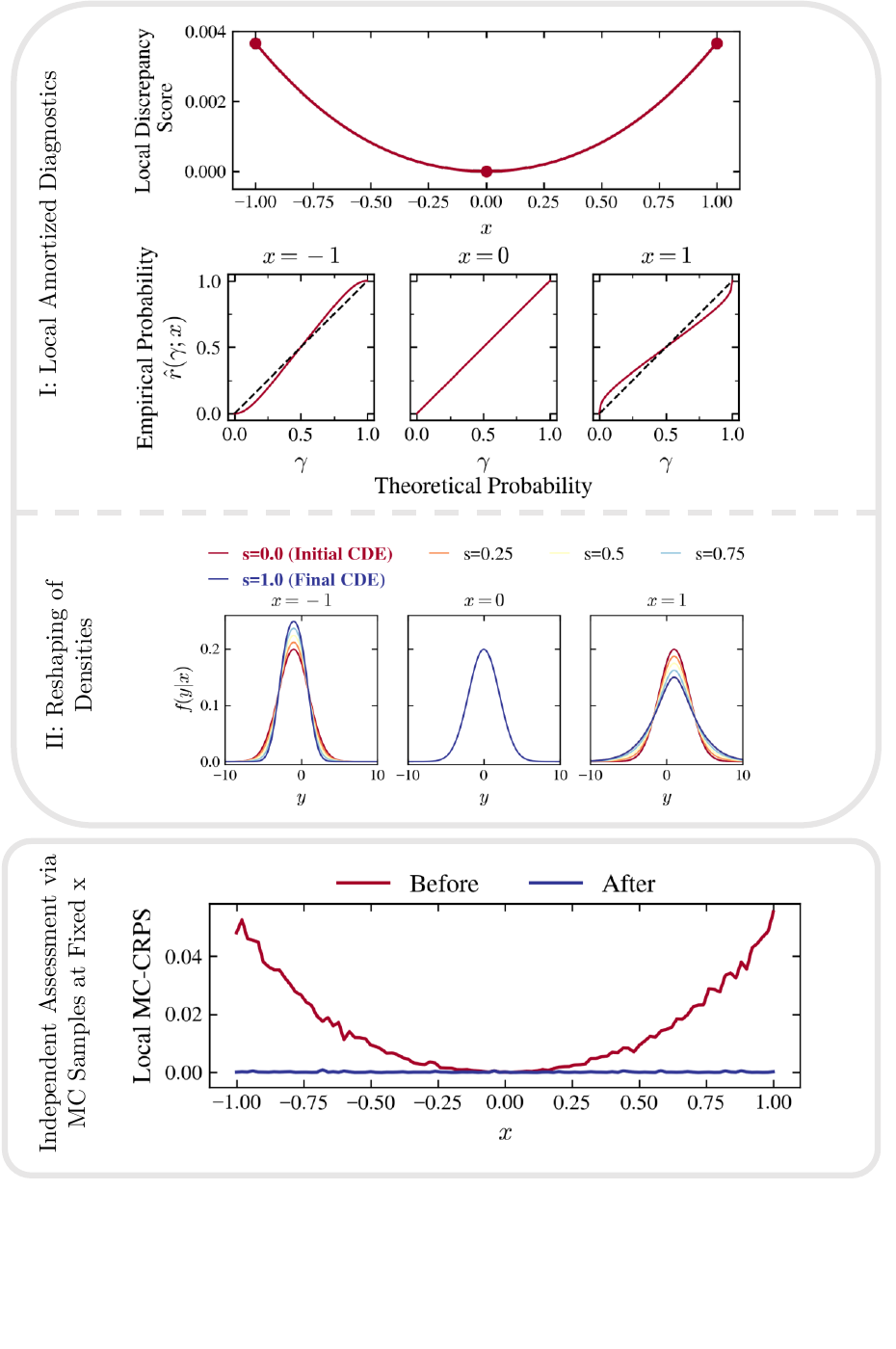}
  
		\caption{\footnotesize {\bf Illustration of LADaR framework: Example 1 kurtotic data.} Initial CDE is Gaussian, but the true distribution is kurtotic.  {\em Top panel I:} Local discrepancy score across the input space (first row) and examples of diagnostic P-P plots (second row). \calpit identifies that the model is {\em over-/under-dispersed} relative to calibration data at $X =−1$ / $X =1$ but well-estimated at $X=0$. The diagnostics define a family of P-P maps for reshaping the initial densities so as to fit the calibration data across the feature space. {\em Top panel II:} Continuous morphing of densities via \calpit, illustrated at the three evaluation points, from the initial Gaussian distributions (\textcolor{red}{\bf red}; $s=0$) to the final distributions (\textcolor{blue}{\bf blue}; $s=1$).  For illustrative purposes, we have included intermediate values of $s$ to show the morphing of distributions. {\em Bottom panel:} Independent assessment of final results by computing a local Monte Carlo version of the continuous ranked probability score (MC-CRPS) at fixed $x$ before and after \calpit.} \label{fig:ex1_kurtotic}
\end{figure}

\section{Example 3: Prediction Sets} \label{app:example_prediction_sets}

\calpit's uniqueness stems from its ability to generate complete CDEs with approximate instance-wise coverage. Additionally, it enables the creation of prediction sets with approximate conditional coverage. Considering the extensive literature on prediction sets, we have included an additional example to demonstrate that prediction sets obtained from \calpit can effectively compete with those derived using methods such as conformal inference or quantile regression. We also include a comparison with normalizing flows, as they have gained popularity for density estimation in the physical sciences. 

In \pz estimation,  multiple widely different distances (redshifts) can be consistent with the observed features (colors) of a galaxy. As mentioned previously, this results in conditional distributions that are multi-modal in parts of the feature space. Motivated by the \pz application, we have modified the two-group example of \citet{feldman2021orthogQR} to have bimodal structure due to limited predictor information. Here the target variable $Y$ depends on three variables: $X_0, X_1, X_2$. Variable $X_0$ indicates group membership but it is not measured; that is, $X_1$ and $X_2$ are our only predictors. The missing membership information results in the CDE $f(y|x_1, x_2)$ being bimodal in the regime $X_{1}>0$ with one branch corresponding to each class. Supplementary material \href{https://lee-group-cmu.github.io/cal-pit-paper/supplementary_material.pdf}{S2}\footnote[1]{Supplementary Materials: \url{https://lee-group-cmu.github.io/cal-pit-paper/supplementary_material.pdf}} details the data-generating process (DGP), and Figure~\ref{fig:ex_1_data} visualizes one random instance of data drawn from  $f(y|x_1, x_2)$ with the ``majority'' and ``minority'' groups displayed as \textcolor{blue}{blue} versus \textcolor{red}{red} points.

We design two experiments for benchmarking \calpit prediction sets against results from conformal inference, quantile regression, and normalizing flows:
\begin{itemize}
 \item Experiment 1 (comparison with conformal inference):  For this experiment, we split a sample of total size $n$ in two halves: the first half is used to train an initial model, and the second half is used for calibration. The empirical coverage of the final prediction sets are computed via  $1000$ MC simulations from the true DGP at each test point on a grid. Test points with coverage within two standard deviations (SD) of the nominal coverage of $1-\alpha=0.9$ based on 100 random realizations are labeled as having ``correct'' coverage. We report the proportion of test points in the feature space with ``under-,'' ``correct,'' and ``over-'' coverage.
  \item  Experiment 2 (comparison with quantile regression and normalizing flows): Here we use the entire sample of size $n$ to compute quantiles or to estimate the conditional density. As above, we use MC simulations on a grid to assess conditional coverage. 
\end{itemize}

The top row of Figure~\ref{fig:ex1_coverage} shows results for Experiment 1. We compare $90\%$ prediction sets for $Y$ using \texttt{Cal-PIT (INT)} and \texttt{Cal-PIT (HPD)} (defined by Equations~\ref{eq:CalPIT_int} and~\ref{eq:CalPIT_HPD}, respectively) with prediction sets from \texttt{Reg-split} \citep{lei2018distribution}, conformalized quantile regression (\texttt{CQR}; \citealt{romano2019conformalizedQR}), and distributional conformal prediction (\texttt{DCP}; \citealt{chernozhukov2021distributional}).  \texttt{Reg-split}
and \texttt{CQR} are trained with XGBoost~\citep{Chen2016XGBoost}. Our \calpit methods use
an initial CDE trained using FlexCode with an XGBoost regressor \citep{izbicki2017converting,Dalmasso2020FlexcodePhotoz} and monotonic neural networks~\citep{Wehenkel2019UMNN} for learning $\widehat r^{\widehat f}(\gamma;\x)$ with binary cross entropy loss. \texttt{DCP} computes a conformal score based on PIT values derived from the same initial CDE as \calpit.  In terms of conditional coverage, all methods improve with increasing sample size, but only \calpit consistently attains the nominal $90\%$ coverage across the feature space for $n\geq2000$. As the data distribution can sometimes be bimodal, the most efficient prediction sets in this feature subspace would not be single intervals (INT), but rather pairs of intervals. We can create such disjoint prediction sets using Highest Predictive Density regions (HPD; see~\ref{app:cal_hpd} for definition).

The bottom row of Figure~\ref{fig:ex1_coverage} shows results for Experiment 2. \texttt{Cal-PIT (INT)} and \texttt{Cal-PIT (HPD)}  reshape a uniform distribution on $\x \in [-5,5]$; hence, there is no need for a separate training set. The \calpit prediction sets are then compared to output from quantile regression (\texttt{QR}; \citealt{koenker1978QR}) trained with XGBoost and a pinball loss,  orthogonal quantile regression (\texttt{OQR}; \citealt{feldman2021orthogQR}) which introduces a penalty on the pinball loss to improve conditional coverage, and normalizing flows (\texttt{NF}). We use the \texttt{PZFlow} \citep{pzflow} implementation of Normalizing Flows which has been optimized to work well out-of-the-box with tabular data and uses Neural Spline Flows \citep{splineFlow1,Spline2,spline3} as the backbone.

\begin{figure}[h!]
    \centering
    \includegraphics[width=0.8\columnwidth]{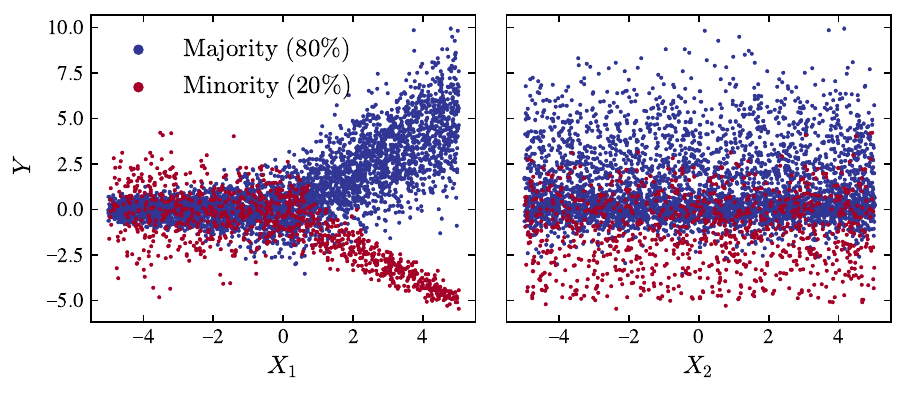}
    \caption{ \small
    Visualization of one random instance of the data used for Example 1. There are two covariates ($X_1, X_2$), and a target variable $Y$. The analytic form of the true data distribution is defined in supplementary material~\href{https://lee-group-cmu.github.io/cal-pit-paper/supplementary_material.pdf}{S2}. The data set consists of two groups with different spreads. $Y$ splits into two branches for $X_{1}>0$;  that is, the true CDE is bimodal in this  region.
    \label{fig:ex_1_data}}
\end{figure}

Figure~\ref{fig:ex1_PDFs}, top row, shows some examples of calibrated CDEs  from \calpit. The estimates reveal that the true conditional density is bimodal for $X_1 > 0$;  thus, the most efficient prediction sets in this feature subspace would be HPD regions. Indeed, \texttt{Cal-PIT (HPD)} yields smaller prediction sets than \texttt{Cal-PIT (INT)}; see Figure~\href{https://lee-group-cmu.github.io/cal-pit-paper/supplementary_material.pdf}{S1} in supplementary material. Because HPD sets can capture the bimodality in the data while intervals cannot, this is a case where \texttt{Cal-PIT (HPD)} has better efficiency. This qualitative insight is only possible because \calpit estimates the entire PDs. Normalizing flows also provide entire CDEs (see Figure~\ref{fig:ex1_PDFs}, bottom row) but can be difficult to train. Indeed, the normalizing flow CDEs generally deviate significantly from the oracle.

\begin{figure}[htb]
\centering
    \begin{subfigure}{\columnwidth}
    \includegraphics[width=0.48\columnwidth]{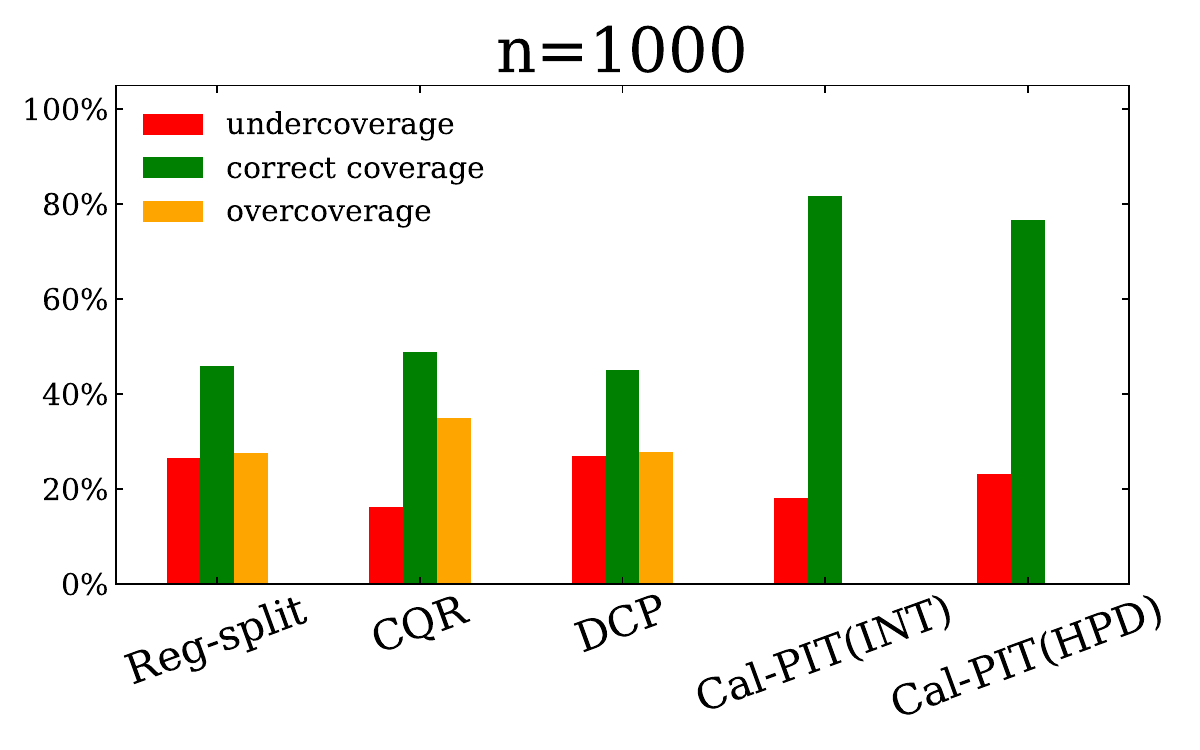}
     \includegraphics[width=0.48\columnwidth]{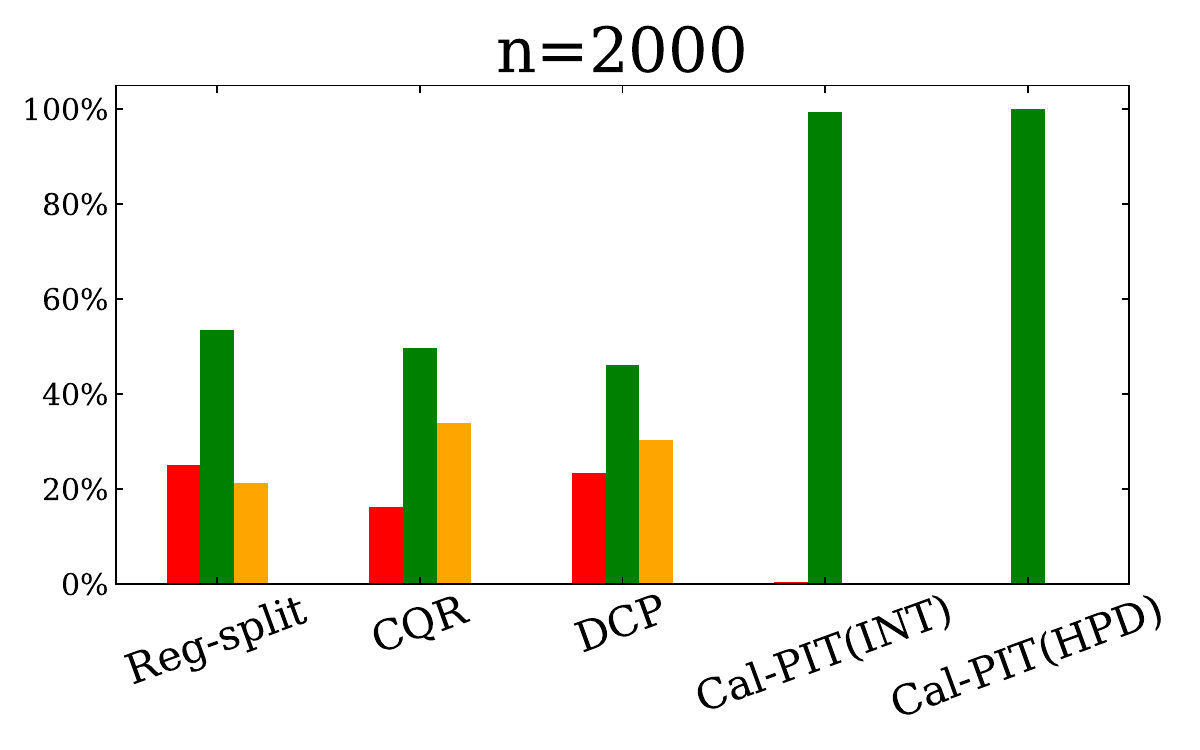}
     \caption{}
     \label{fig:ex_1_coverage_a}
    \end{subfigure}
    
    \begin{subfigure}{\columnwidth}
    \includegraphics[width=0.48\columnwidth]{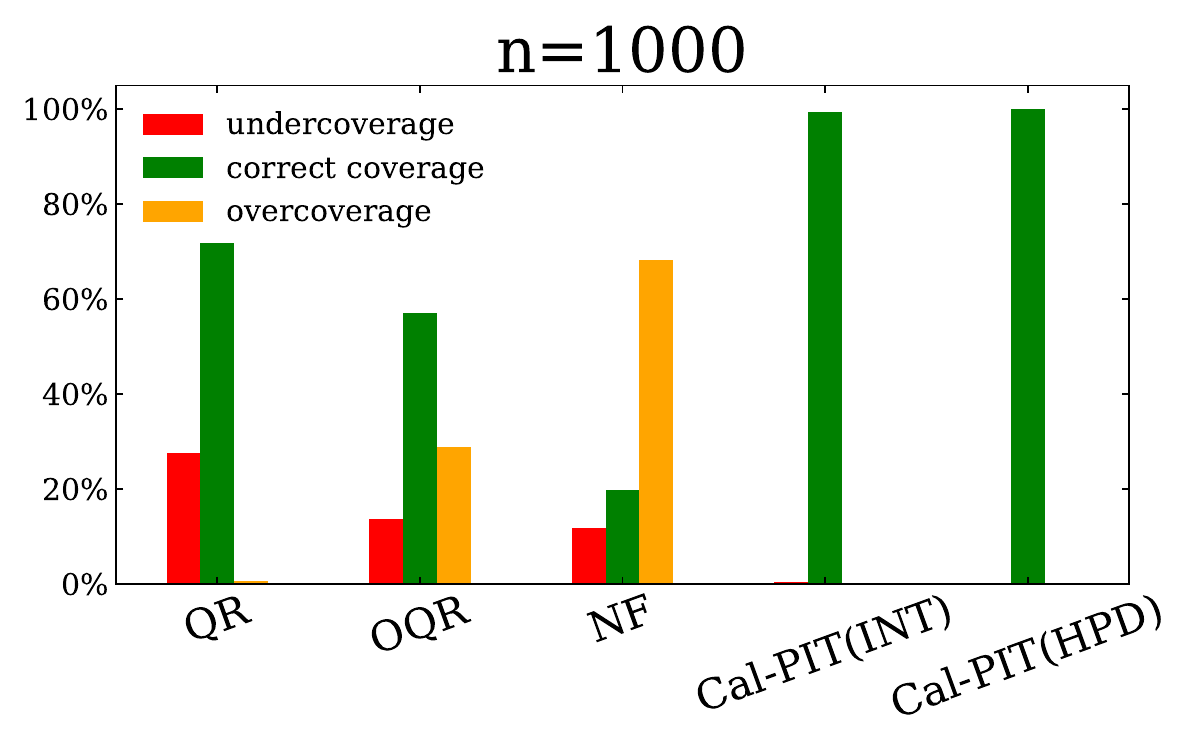}
    \includegraphics[width=0.48\columnwidth]{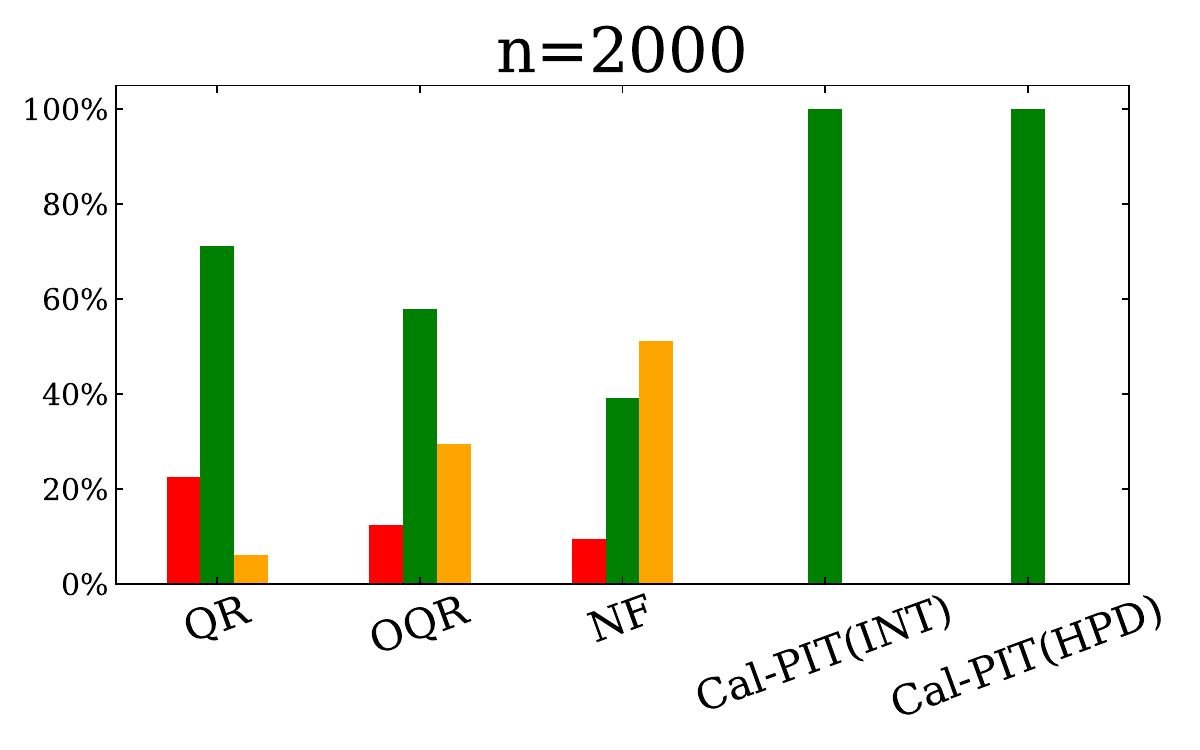}
    \caption{}
    \label{fig:ex_1_coverage_b}
    \end{subfigure}
 
	\caption{\small The proportion of test points with correct conditional coverage for (a) ``Experiment 1'' with state-of-the-art conformal inference methods, using data of total size $n$ split into a train and a calibration set, and (b) ``Experiment 2'' with quantile regression and normalizing flow approaches, which use all data for training. See text for details. Only \calpit  consistently attains the nominal $90\%$ coverage across the feature space with increasing sample size $n$. 
	} 
	\label{fig:ex1_coverage}
\end{figure}

\begin{figure}[htb]
\begin{subfigure}{\columnwidth}
    \includegraphics[width=\columnwidth, right]{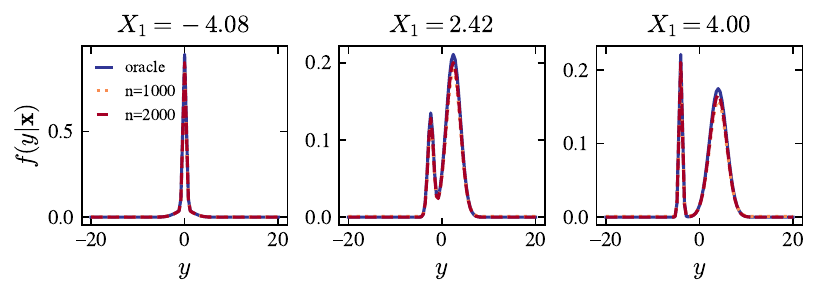}
    \caption{CDEs from \calpit}
\end{subfigure}
\vspace{1cm}
\begin{subfigure}{\columnwidth}
    \includegraphics[width=\columnwidth, right]{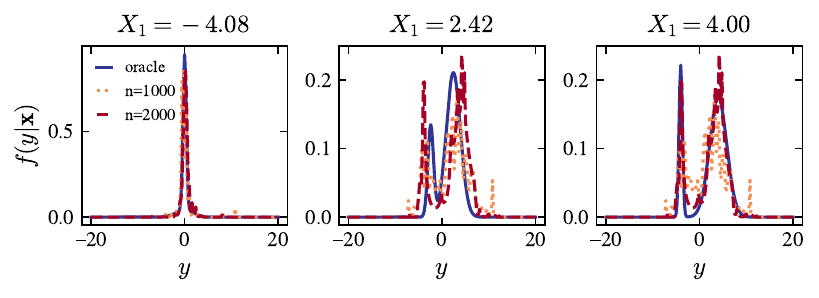}
    \caption{CDEs from Normalizing Flows}
\end{subfigure}
	\caption{ \small CDEs at three different values of $X_1$ ($X_2=0$) for (a) \calpit and (b) Normalizing Flows for  Example 3. The results for $n=1000$ and $n=2000$ are compared to the ``oracle'' probability density functions. For both sample sizes, the \calpit CDEs are close to the oracle. Normalizing flow CDEs, on the other hand, are harder to train and a standard implementation can deviate significantly from the oracle.}
\label{fig:ex1_PDFs}
\end{figure}
\clearpage

\section{Local CRPS Scores}\label{appendix:local_CRPS}
The conditional expectation of the CRPS loss given $\X = \x$ is
$$  \E \left[ L_{\mathrm{CRPS}}(\widetilde{f};\X,Y)\big|\x\right] = \E \left[\int_{-\infty}^{\infty} \left( \widetilde{F}(t|\x) -F(t|\x) + F(t|\x) - \I(Y \leq t) \right)^2 dt \Big|\x \right].
$$
By expanding the square and by changing the order of expectation and integration, we have:
\begin{eqnarray*}
    \E \left[ L_{\mathrm{CRPS}}(\widetilde{f};\X,Y)|\x\right] &=\E \left[\int_{-\infty}^{\infty} \left( \widetilde{F}(t|\x) -F(t|\x) \right)^2 dt \mid \x\right] \\
    &+2\int_{-\infty}^{\infty} \left( \widetilde{F}(t|\x) -F(t|\x)\right)\E \left[\left(F(t|\x) - \I(Y \leq t) \right)dt|\x\right] \\
    &+
    \int_{-\infty}^{\infty} \E \left[ \left(F(t|\x) - \I(Y \leq t) \right)^2|\x\right] dt. 
\end{eqnarray*}
Note that:
\begin{itemize}
    \item The first term represents the squared distance between $\widetilde{F}$ and $F$ and is minimized when $\widetilde{F}(\cdot|\x) = F(\cdot|\x)$.
    \item 
    The second term equals zero,
    $$\E \left[ F(t|\x) - \I(Y \leq t) \big|\x \right] = F(t|\x)-\E \left[  \I(Y \leq t) \big|\x \right]=F(t|\x)-F(t|\x)=0.$$
    \item The third term is a constant that does not depend on $\widetilde{F}$.
\end{itemize}
Thus, 
\begin{eqnarray*}
    \E \left[ L_{\mathrm{CRPS}}(\widetilde{f};\X,Y)\big|\x\right] &= \int_{-\infty}^{\infty} \left( \widetilde{F}(t|\x) -F(t|\x) \right)^2 dt + K \\
    &\approx  \int \left(\widetilde F(t|\x) - \frac{1}{B} \sum_{b=1}^B I(Y_b<t) \right)^2dt + K \\
   & =  L_{\mathrm{MC-CRPS}}(\widetilde{f};\x, f)+ K ,
\end{eqnarray*}
where $K$ does not depend on $\widetilde{F}$.

\section{\texttt{Cal-PIT} (HPD) and \texttt{Cal-HPD}} \label{app:cal_hpd}

 Here we describe two approaches to deriving prediction sets (instead of prediction intervals) from an estimate of the conditional distribution function $f(y|\x)$.

\subsection{\texttt{Cal-PIT} (HPD)}
\texttt{Cal-PIT} can also be used to compute Highest Predictive Density regions (HPDs)  instead of prediction intervals.
The oracle (1-$\alpha$)-level HPD set is defined as
$$\textrm{HPD}_\alpha(\x)=\{y:   f(y|\x)\geq  t_{\x,\alpha}\},$$
	  where $ t_{\x,\alpha}$ is such that 
	  $\int_{y \in \textrm{HPD}_\alpha(\x)}  f(y|\x)dy=1-\alpha$.
	   HPDs are the smallest prediction sets that have coverage $1-\alpha$, and thus they may be more  precise (smaller set size) than quantile-based intervals,  while maintaining the conditional coverage at the nominal level (see \ref{app:example_prediction_sets}
       for an example with a bimodal predictive distribution). 

The \texttt{Cal-PIT} estimate of 
$\textrm{HPD}_\alpha(\x)$ is given by
 \begin{equation} \label{eq:CalPIT_HPD}
     C_\alpha(\x)=\{y: \widetilde f(y|\x)\geq \widetilde t_{\x,\alpha}\},
\end{equation}
	  where $\widetilde t_{\x,\alpha}$ is such that 
	  $\int_{y \in C_\alpha(\x)} \widetilde f(y|\x)dy=1-\alpha$
	  and $\widetilde f$ is the 
	 \texttt{Cal-PIT} calibrated CDE (Algorithm \ref{alg:CalPIT}).\\

\subsection{\texttt{Cal-HPD}}

 Alternatively, one can directly use HPD values, defined as 
\begin{eqnarray*}
\hat H(y; \x) := \int_{\left\{y: \hat{f}(y'|\x) \leq \hat{f}(y|\x)\right\}} \hat{f}(y'|\x) dy',
\end{eqnarray*}
 to recalibrate  HPD prediction sets  (rather than using PIT values). The idea is to 
estimate the local HPD coverage at each $\x$,
$h^{\hat f}(\gamma;\x) := \pr(\hat H(Y;\x) \leq \gamma | \x),$ 
by regression, analogous to  estimating the PIT-CDF in \texttt{Cal-PIT}.
Let  $\hat h^{\hat f}(\gamma;\x)$ be such an estimate. The recalibrated $(1-\alpha)$-level HPD set at a location $\x$ is given by the  $(1-\alpha^*(\x))$-level HPD set of the original density $\hat f(y|\x)$, where 
$\alpha^*(\x)$ is such that $\hat h^{\hat f}(\alpha^*(\x);\x) = \alpha$. This framework however does  not  yield full CDEs. Moreover, although  the approach corrects HPD sets, aiming for conditional coverage, the constructed sets will not be optimal if the initial model $\widehat f$ is far from the true data generating process $f$.\\

In  Example 3 (\ref{app:example_prediction_sets}), we only report results for \texttt{Cal-PIT}(INT) and \texttt{Cal-PIT}(HPD); we do not report results for \texttt{Cal-HPD}.\\

\section{Theoretical Properties of \texttt{Cal-PIT}}
\label{sec:theory}

We here describe the assumptions needed for Theorem \ref{thm:performance_PD}, 
and provide convergence rates. We also prove  
that \calpit intervals achieve asymptotic conditional validity even if the initial CDE $\hat{f}$ is not consistent.
The following results are conditional on  $\hat{f}$; all uncertainty refers to the calibration sample. 
We assume in Theorem \ref{thm:performance_PD} that the true distribution of $Y|\x$ and its initial estimate are continuous, and that
  $\widehat F$ places its mass on a region that is at least as large as  that of $F$:
\begin{Assumption}[Continuity of the cumulative distribution functions]
	\label{assump:continuity}
	For every $\x \in \mathcal{X}$, $\widehat F(\cdot |\x)$ and $F(\cdot |\x)$ are  continuous  functions.
\end{Assumption}

\begin{Assumption}[$\widehat F$ dominates $F$]
    \label{assump:dominates}
    	For every $\x \in \mathcal{X}$, $\widehat F(\cdot |\x)$ dominates $F(\cdot|\x)$.
\end{Assumption}

We also assume that  $F(\cdot|\x)$ cannot place too much mass in regions where the initial estimate 
$\widehat F(\cdot|\x)$
 places little mass:
\begin{Assumption}[Bounded density]\label{assump:bounded}
There exists $K>0$ such that, for every $\x \in \mathcal{X}$, the Radon-Nikodym derivative of $F(\cdot | \x)$ with respect to $\widehat F(\cdot | \x)$ is bounded above by   $K$.
\end{Assumption}

To provide rates of convergence for the recalibrated CDE, we will in addition assume that the regression method converges at a rate $O(n^{-\kappa})$:
\begin{Assumption}[Convergence rate of the regression method]\label{assump:convergence_rate} \label{assump:mse}
The regression method used to estimate $r^{\widehat f}$ is such that its convergence rate is given by
$$\E \left[ \int \int \left(\widehat r^{\widehat f}(\gamma;\x)-r^{\widehat f}(\gamma;\x) \right)^2 d\gamma dP(\x) \right]=O\left(\frac{1}{n^\kappa}\right)$$
for some $\kappa>0$.
\end{Assumption}

Many methods  satisfy Assumption \ref{assump:convergence_rate} for some value
$\kappa$, which is typically rated to the dimension of $\mathcal{X}$ and the smoothness of the true regression $r$ (see for instance \citealt{gyorfi2002distribution}).

Under these assumptions, we can derive the rate of convergence for $\widetilde F$:
\begin{Corollary}[Convergence rate of recalibrated CDE]
\label{cor:rate}
Under Assumptions  \ref{assump:continuity}, \ref{assump:dominates}, \ref{assump:bounded} and \ref{assump:convergence_rate},
\begin{equation}
\E \left[ \int \int \left(\widetilde F(y|\x)-F(y|\x) \right)^2  dP(y,\x)\right]=O\left(\frac{1}{n^\kappa}\right).
\end{equation}
\end{Corollary}

Next, we show that
 with an uniformly consistent regression estimator $\hat r^{\hat f}(\gamma;\x)$ (see \citealt{bierens1983uniform,hardle1984uniform,liero1989strong,girard2014uniform} for some examples), \calpit intervals achieve asymptotic conditional validity, even if the initial CDE $\hat{f}(y|\x)$ is not consistent.

\begin{Assumption}[Uniform consistency of the regression estimator]
	\label{assump:uniform_consistency}
	The regression estimator  is such that
	$$\sup_{\x \in \mathcal{X},\gamma \in [0,1]} | \widehat r^{\widehat f}(\gamma;\x)- r^{\widehat f}(\gamma;\x)|  \xrightarrow[n \longrightarrow\infty]{\enskip \textrm{a.s.} \enskip}  0,$$
	where the convergence is with respect to the calibration set $\mathcal{D}$ only; $\widehat f$ is fixed.
\end{Assumption}

\begin{thm}[Consistency and conditional coverage of 	\calpit intervals]
	\label{thm:consistency}
	Let $C^*_\alpha(\x)=\left[F^{-1}(0.5\alpha|\x); 
F^{-1}(1-0.5\alpha|\x)\right]$  be the oracle prediction band, and let $C^n_\alpha(\x)$ denote the \calpit interval. 
	Under Assumptions \ref{assump:continuity}, \ref{assump:dominates} and \ref{assump:uniform_consistency},
\begin{equation}
     \lambda\left(C_\alpha^n(\X) \Delta C_\alpha^*(\X)\right) \xrightarrow[n \longrightarrow\infty]{\enskip \textrm{a.s.} 
     \enskip}  0,
\end{equation}
	where $\lambda$ is the Lebesgue measure in $\mathbb{R}$ and $ \Delta$ is the symmetric difference between two sets. It follows that
	$C_\alpha^n(\X)$ has  asymptotic conditional coverage of
	$1-\alpha$ \citep{lei2018distribution}.
\end{thm}
See \ref{sec:hpds} for theoretical results for \texttt{Cal-PIT (HPD)}.

\section{Proofs} \label{app:proofs}

\begin{Lemma} 
\label{lemma:equality_cumulative}
Let $G$ and $H$ be two cumulative distribution functions   such that $G$ dominates $H,$
and let $\mu_G$ and
$\mu_H$ be their associated measures over $\mathbb{R}$.
Then, for every fixed $y \in \mathbb{R}$,
 $$\mu_H\left(\{y' \in \mathbb{R}:y'\leq y\}\right)=  
\mu_H\left(\{y'  \in \mathbb{R}:G(y')\leq G(y)\}\right).$$
\end{Lemma} 

\begin{proof}
Fix $y \in \mathbb{R}$ and let $A=\{y'  \in \mathbb{R}:y'\leq y\}$
and $B=\{y'  \in \mathbb{R}:G(y')\leq G(y)\}$.
Because $A \subseteq B$,
\begin{equation}
\label{eq:larger}
\mu_H(A)\leq   
\mu_H(B).
\end{equation}

We note that
$\mu_G(B \cap A^c)=0$. From this and the assumption that  $G$ dominates $H$, we conclude that
$\mu_H(B \cap A^c)=0$. It follows that
\begin{eqnarray}
\label{eq:smaller}
\mu_H(B)&=\mu_H(B \cap A)+\mu_H(B \cap A^c)  \leq \mu_H(A)+0  \nonumber \\
&=\mu_H(A).
\end{eqnarray}
From  Equations \ref{eq:larger} and \ref{eq:smaller}, we conclude that 
$\mu_H(A)=   
\mu_H(B)$.

\end{proof}

\begin{Lemma} \label{lemma:relationshipFandR}
Fix $y \in \mathbb{R}$
	and let $\gamma:= \widehat F(y|\x)$. Then, under Assumptions \ref{assump:continuity} and \ref{assump:dominates},
$\widetilde F(y|\x)=\widehat r^{\widehat f}(\gamma;\x)$ and
$ F(y|\x)= r^{\widehat f}(\gamma;\x)$.
\end{Lemma}

\begin{proof}
	We note that $\gamma= \widehat F(y|\x)$ implies that $y=\widehat F^{-1}(\gamma|\x)$. It follows then by construction,
	\begin{eqnarray*}
	\tilde F(y|\x)=    
	\tilde F\left(\widehat F^{-1}(\gamma|\x)|\x\right)=\widehat r^{\widehat f}(\gamma;\x).
	\end{eqnarray*} 
	Moreover,
	\begin{eqnarray*}
	F(y|\x)&=\P(Y \leq y|\x) \\
	&=\P\left(\widehat F(Y|\x) \leq \widehat F(y|\x)|\x\right) & \\  &\scriptstyle(\textrm{Assumption~\ref{assump:dominates}~and~ Lemma~\ref{lemma:equality_cumulative}}) \nonumber &\\
	&=\P\left(\textrm{PIT}(Y;\x) \leq \widehat F(y|\x)|\x\right)\\
	&=\P\left(\textrm{PIT}(Y;\x) \leq  \gamma|\x\right)  &\\
	&=r^{\widehat f}(\gamma;\x), &
	\end{eqnarray*}
	which concludes the proof.
\end{proof}

\begin{proof}[Proof of Theorem \ref{thm:performance_PD}]
Consider the change of variables $\gamma=\widehat F(y|\x)$, so that $d \gamma=\widehat f(y|\x)dy$.
Lemma \ref{lemma:relationshipFandR}  implies that $\widetilde F(y|\x)=\widehat r^{\widehat f}(\gamma;\x)$ and
$ F(y|\x)= r^{\widehat f}(\gamma;\x)$. It follows from that and Assumption \ref{assump:bounded} that
\begin{eqnarray*}
    \int \int &\left(\widetilde F(y|\x)-F(y|\x) \right)^2  dP(y,\x)\\
    &\leq K \int \int  \left(\widetilde F(y|\x)-F(y|\x) \right)^2 \widehat f(y|\x) dyP(\x) \\
        &=K  \int \int \left(\widehat r^{\widehat f}(\gamma;\x)-r^{\widehat f}(\gamma;\x) \right)^2 d\gamma dP(\x),
\end{eqnarray*}
which concludes the proof.
\end{proof}

\begin{proof}[Proof of Corollary \ref{cor:rate}]
Follows directly from  Assumption \ref{assump:convergence_rate}
and Theorem \ref{thm:performance_PD}.
\end{proof}

\begin{proof}[Proof of Theorem \ref{thm:consistency}]
From Lemma \ref{lemma:relationshipFandR}, 
\begin{eqnarray*}
	\sup_{\x \in \mathcal{X},y \in \mathbb{R}} &| \tilde F(y|\x)- F(y|\x)| \\
 =\sup_{\x \in \mathcal{X},\gamma \in [0,1]} &| \widehat r^{\widehat f}(\gamma;\x)- r^{\widehat f}(\gamma;\x)| \xrightarrow[n \longrightarrow\infty]{\enskip \textrm{a.s.} \enskip}  0,
	\end{eqnarray*}
	where the last step follows from  Assumption \ref{assump:uniform_consistency}.
	It then follows from Assumption \ref{assump:continuity} that
	$$\sup_{\x \in \mathcal{X},\gamma \in [0,1]} | \tilde F^{-1}(\gamma|\x)- F^{-1}(\gamma|\x)|  \xrightarrow[n \longrightarrow\infty]{\enskip \textrm{a.s.} \enskip}  0,$$
	and, in particular,
	$$\sup_{\x \in \mathcal{X},\alpha \in \{.5\alpha,1-.5\alpha\}} | \tilde F^{-1}(\alpha|\x)- F^{-1}(\alpha|\x)|  \xrightarrow[n \longrightarrow\infty]{\enskip \textrm{a.s.} \enskip}  0,$$
	from which the conclusion of the theorem follows.
\end{proof}

\subsection{Theory for  \texttt{Cal-PIT} HPD sets}
\label{sec:hpds}

For every $\x \in \mathcal{X}$,
	  let $C_\alpha(\x)=\{y: \widetilde f(y|\x)\geq \widetilde t_{\x,\alpha}\}$,
	  where $\widetilde t_{\x,\alpha}$ is such that 
	  $\int_{y \in C_\alpha(\x)} \widetilde f(y|\x)dy=1-\alpha $ be the \texttt{Cal-PIT} HPD-set. Similarly, let 
	  $\textrm{HPD}_\alpha(\x)=\{y:   f(y|\x)\geq  t_{\x,\alpha}\}$,
	  where $ t_{\x,\alpha}$ is such that 
	  $\int_{y \in \textrm{HPD}_\alpha(\x)}  f(y|\x)dy=1-\alpha $  be the true HPD-set.
The next theorem shows that if the probabilistic classifier is well estimated, then \texttt{Cal-PIT} HPD sets are exactly equivalent to oracle HPD sets.

\begin{thm}[Fisher consistency 	\texttt{Cal-PIT} HPD-sets] 
	\label{thm:Fisherconsistency}
	 Fix $\x \in \mathcal{X}$. If  $\widehat r(\gamma;\x)= r(\gamma;\x)$ for every   $\gamma \in [0,1]$, 
	$ C_\alpha(\x)=
	\textrm{HPD}_\alpha(\x)$
	and $\P(Y \in C_\alpha(\X)|\x)=1-\alpha$.
\end{thm}

\begin{proof}[Proof of Theorem \ref{thm:Fisherconsistency}] Fix $y \in \mathbb{R}$ and
	let $\gamma=\widehat F(y|\x)$, so that $y=\widehat F^{-1}(\gamma|\x)$. It follows that
	\begin{eqnarray*}
	\tilde F(y|\x)&=
	\tilde F\left(\widehat F^{-1}(\gamma|\x)|\x\right)=\widehat r(\gamma;\x)=
	r(\gamma;\x)\\
	&=\P \left(\widehat F(Y|\x) \leq \widehat F(y|\x) | \x,\gamma \right)=
	\P \left(Y  \leq y | \x,\gamma \right)\\
	&=F(y|\x),
	\end{eqnarray*}
	and therefore $\tilde f(y|\x)=f(y|\x)$ for almost every $y \in \mathbb{R}$. It follows that $C_\alpha(\x)=\textrm{HPD}_\alpha(\x)$. 
	The claim about conditional coverage follows from the definition of the HPD.
\end{proof}

\subsection{Further Details on Experiments} 

We refer the reader to the online supplementary materials for details on the training of the regression model to learn the PIT-CDF function in our experiments, further remarks on Example 3 (prediction sets) results, and a description of the synthetic data generation and the training of the initial ConvMDN model in Example 2.